\documentclass[twoside,11pt]{article}

\usepackage{jmlr2e}
\usepackage{hyperref}
\usepackage{algorithm}
\usepackage{algorithmic}
\usepackage{multirow}
\usepackage{amsmath}
\usepackage{xspace} 
\usepackage{colortbl}
\usepackage{caption}
\usepackage{subcaption}
\usepackage{bm, graphicx}
\newtheorem{prop}{Proposition}

 \newcommand{\w}{{\boldsymbol \theta}}
 \newcommand{\tr}{\top}     % \intercal  , \top,  \mathsf{T}
 \newcommand{\x}{{\boldsymbol \phi}}
 \newcommand{\bs}{\boldsymbol}

\jmlrheading{17}{2016}{1-40}{11/15; Revised 7/16}{8/16}{Harm van Seijen, A. Rupam Mahmood, Patrick M. Pilarski, Marlos C. Machado and Richard S. Sutton}

\ShortHeadings{True Online Temporal-Difference Learning}{van Seijen, Mahmood, Pilarski, Machado and Sutton}
\firstpageno{1}

\begin{document}
\title{True Online Temporal-Difference Learning}

\author{\name{Harm van Seijen$^{\dagger\ddagger}$} \email{harm.vanseijen@maluuba.com}\\
       \name{A. Rupam Mahmood$^\dagger$} \email{ashique@ualberta.ca}\\
       \name{Patrick M. Pilarski$^\dagger$} \email{patrick.pilarski@ualberta.ca}\\
       \name{Marlos C. Machado$^\dagger$} \email{machado@ualberta.ca}\\
       \name{Richard S. Sutton$^\dagger$} \email{sutton@cs.ualberta.ca} \\
       \\
       \addr $^\dagger$Reinforcement Learning and Artificial Intelligence Laboratory\\
       University of Alberta\\
       2-21 Athabasca Hall, Edmonton, AB\\
       Canada, T6G 2E8\\
       \\
        \addr $^\ddagger$Maluuba Research\\
        2000 Peel Street, Montreal, QC\\
        Canada, H3A 2W5}

\editor{George Konidaris}

\maketitle

\begin{abstract}%
The temporal-difference methods TD($\lambda$) and Sarsa($\lambda$) form a core part of modern reinforcement learning. Their appeal comes from their good performance, low computational cost, and their simple interpretation, given by their forward view. Recently, new versions of these methods were introduced, called true online TD($\lambda$) and true online Sarsa($\lambda$), respectively \citep{vanseijen:icml14}. Algorithmically, these true online methods only make two small changes to the update rules of the regular methods, and the extra computational cost is negligible in most cases. However, they follow the ideas underlying the forward view much more closely. In particular, they maintain an exact equivalence with the forward view at all times, whereas the traditional versions only approximate it for small step-sizes. We hypothesize that these true online methods not only have better theoretical properties, but also dominate the regular methods empirically. In this article, we put this hypothesis to the test by performing an extensive empirical comparison. Specifically, we compare the performance of true online TD($\lambda$)/Sarsa($\lambda$) with  regular TD($\lambda$)/Sarsa($\lambda$) on random MRPs, a real-world myoelectric prosthetic arm, and a domain from the Arcade Learning Environment. We use linear function approximation with tabular, binary, and non-binary features. Our results suggest that the true online methods  indeed dominate  the regular methods. Across all domains/representations the learning speed of the true online methods are often better, but never worse than that of the regular methods. An additional advantage is that no choice between traces has to be made for the true online methods.
Besides the empirical results, we provide an in-depth analysis of the theory behind true online temporal-difference learning. In addition, we show that new true online temporal-difference methods can be derived by making changes to the online forward view and then rewriting the update equations.
\end{abstract}

\section{Introduction}

Temporal-difference (TD) learning is a core learning technique in modern reinforcement learning \citep{sutton:ml88, kaelbling:jair96, sutton:book98, szepesvari:book10}. One of the main challenges in reinforcement learning is to make predictions, in an initially unknown environment, about the (discounted) sum of future rewards, the return, based on currently observed feature values and a certain behaviour policy. With TD learning it is possible to learn good estimates of the expected return quickly by bootstrapping from other expected-return estimates. TD($\lambda$) \citep{sutton:ml88} is a popular TD algorithm that combines basic TD learning with eligibility traces to further speed learning. The popularity of TD($\lambda$) can be explained by its simple implementation, its low-computational complexity and its conceptually straightforward interpretation, given by its forward view. The forward view of TD($\lambda$) states that the estimate at each time step is moved towards an update target known as the  $\lambda$-return, with $\lambda$ determining the fundamental trade-off between bias and variance of the update target. This trade-off has a large influence on the speed of learning and its optimal setting varies from domain to domain. The ability to improve this trade-off by adjusting the value of $\lambda$ is what underlies  the performance advantage of eligibility traces.

Although the forward view provides a clear intuition, TD($\lambda$) closely approximates the forward view only for appropriately small step-sizes. 
Until recently, this was considered an unfortunate, but unavoidable part of the theory behind TD($\lambda$). This changed with the introduction of true online TD($\lambda$) \citep{vanseijen:icml14}, which computes exactly the same weight vectors as the forward view at any step-size. This gives true online TD($\lambda$) full control 
over the bias-variance trade-off. In particular, true online TD(1) can achieve fully unbiased updates. Moreover, true online TD($\lambda$) only requires small modifications to the TD($\lambda$) update equations, and the extra computational cost is negligible in most cases.

We hypothesize that true online TD($\lambda$), and its control version true online Sarsa($\lambda$), not only have better theoretical properties than their regular counterparts, but also dominate them empirically. We test this hypothesis by performing an extensive empirical comparison between true online TD($\lambda$), regular TD($\lambda$) (which is based on accumulating traces), and the common variation based on replacing traces. In addition, we perform comparisons between true online Sarsa($\lambda$) and Sarsa($\lambda$) (with accumulating and replacing traces). The domains we use include random Markov reward processes, a real-world myoelectric prosthetic arm, and a domain from the Arcade Learning Environment \citep{bellemare:jair13}. The representations we consider range from tabular values to linear function approximation with binary and non-binary features. 

Besides the empirical study, we provide an in-depth discussion on the theory behind true online TD($\lambda$). This theory is based on a new online forward view. The traditional forward view, based on the $\lambda$-return, is inherently an offline forward view meaning that updates only occur at the end of an episode, because the $\lambda$-return requires data up to the end of an episode. We extend this forward view to the online case, where updates occur at every time step, by using a bounded version of the $\lambda$-return that grows over time.  Whereas TD($\lambda$) approximates the traditional forward view only at the end of an episode, we show that TD($\lambda$) approximates this new online forward view at all time steps. True online TD($\lambda$) is equivalent to this new online forward view at all time steps. We prove this by deriving the true online TD($\lambda$) update equations directly from the online forward view update equations. This derivation forms a blueprint for the derivation of other true online methods. By making variations to the online forward view and following the same derivation as for true online TD($\lambda$), we derive several other true online methods.

This article is organized as follows. We start by presenting the required background in Section 2. Then, we present the new online forward view in Section 3, followed by the presentation of true online TD($\lambda$) in Section 4. Section 5 presents the empirical study. Furthermore, in Section 6, we present several other true online methods. In Section 7, we discuss in detail related papers. Finally, Section 8 concludes.

\section{Background}
\label{background}

Here, we present the main learning framework. As a convention, we indicate scalar-valued random variables by capital letters (e.g., $S_t$, $R_t$), vectors by bold lowercase letters (e.g.,  $\w$, $\x$), functions by non-bold lowercase letters (e.g., $v$), and sets by calligraphic font (e.g., $\mathcal{S}$, $\mathcal{A}$).\footnote{An exception to this convention is the TD error, a scalar-valued random variable that we indicate by $\delta_t$.}

\subsection{Markov Decision Processes}
\label{MDPs}

Reinforcement learning (RL) problems are often formalized as \emph{Markov decision processes} (MDPs), which can be described as 5-tuples of the form $\langle \mathcal{S}, \mathcal{A}, p, r, \gamma \rangle$, where $\mathcal{S}$ indicates the set of all states; $\mathcal{A}$ indicates the set of all actions; $p(s'|s,a)$ indicates the probability of a transition to state $s' \in \mathcal{S}$, when action $a \in \mathcal{A}$ is taken in state $s \in \mathcal{S}$;  $r(s,a,s')$ indicates the expected reward for a transition from state $s$ to state $s'$ under action $a$; the discount factor $\gamma$ specifies how future rewards are weighted with respect to the immediate reward. 

Actions are taken at discrete time steps $t = 0,1,2,...$ according to a \emph{policy} $\pi: \mathcal{S} \times \mathcal{A} \rightarrow [0,1]$, which defines for each action the selection probability conditioned on the state. 
The \emph{return} at time $t$ is defined as the discounted sum of rewards, observed after $t$:
$$G_t := R_{t+1} + \gamma\,R_{t+2} + \gamma^2\,R_{t+3}+... = \sum_{i=1}^\infty\,\gamma^{i-1}\, R_{t+i}\thinspace,$$
where $R_{t+1}$ is the reward received after taking action $A_t$ in state $S_t$. Some MDPs contain special states called \emph{terminal states}. After reaching a terminal state, no further reward is obtained and no further state transitions occur. Hence, a terminal state can be interpreted as a state where each action returns to itself with a reward of 0. An interaction sequence from the initial state to a terminal state is called an \emph{episode}.

Each policy $\pi$ has a corresponding state-value function $v_{\pi}$, which maps any state $s \in \mathcal{S}$ to the expected value of the return from that state, when following policy $\pi$:
$$v_{\pi}(s) := \mathbb{E}\{ G_t \,|\, S_t = s, \pi \}\thinspace.$$
In addition, the action-value function $q_{\pi}$ gives the expected return for policy $\pi$, given that action $a \in \mathcal{A}$ is taken in state $s \in \mathcal{S}$:
$$q_{\pi}(s,a) := \mathbb{E}\{ G_t \,|\, S_t = s, A_t, = a, \pi \}\thinspace.$$
Because no further rewards can be obtained from a terminal state, the state-value and action-values for a terminal state are always 0.

There are two tasks that are typically associated with an MDP. First, there is the task of determining (an estimate of) the value function $v_\pi$ for some given policy $\pi$. The second, more challenging task is that of determining (an estimate of) the optimal policy $\pi_*$, which is defined as the policy whose corresponding value function has the highest value in each state:
$$v_{\pi_*}(s) :=  \max_\pi  \,v_\pi(s)\,,  \qquad\mbox{for each } s \in \mathcal{S}\,.$$
In RL, these two tasks are considered under the condition that the reward function $r$ and the transition-probability function $p$ are unknown. Hence, the tasks have to be solved using samples obtained from interacting  with the environment.

\subsection{Temporal-Difference Learning}

Let's consider the task of learning an estimate $V$ of the value function $v_{\pi}$ from samples, where $v_\pi$ is being estimated using linear function approximation. That is, $V$ is the inner product between a feature vector  $\x(s) \in \mathbb{R}^n$ of $s$, and a weight vector $\w \in \mathbb{R}^n$:
$$V(s, \w) =  \w^\tr \x(s)\,.$$
If $s$ is a terminal state, then by definition $\x(s) := {\bs 0}$, and hence $V(s, \w) = 0$.

We can formulate the problem of estimating $v_\pi$ as an error-minimization problem, where the error is a weighted average of the squared difference between the value of a state and its estimate:
$$E(\w) := \frac{1}{2} \sum_{i} d_\pi(s_i) \Big( v_{\pi}(s_i) - \w^\tr\x(s_i) \Big)^2\,,$$
with $d_\pi$ the stationary distribution induced by $\pi$. 
The above error function can be minimized by using stochastic gradient descent while sampling from the stationary distribution, resulting in the following update rule:
$$\w_{t+1} =  \w_t - \alpha \frac{1}{2} \nabla_\theta \Big( v_\pi (S_t) -  \w^\tr\x_t \Big)^2\,,$$
using $\x_t$ as a shorthand for $\x(S_t)$. The parameter $\alpha$ is called the \emph{step-size}. Using the chain rule, we can rewrite this update as:
\begin{eqnarray*}
\w_{t+1}&=& \w_t + \alpha \Big(  v_\pi (S_t) - \w^\tr\x_t \Big) \nabla_\theta (\w^\tr\x_t)\,,\\
&=& \w_t + \alpha \Big(  v_\pi (S_t) -  \w^\tr\x_t \Big)\x_t\,.
\end{eqnarray*}
Because $v_\pi$ is in general unknown, an estimate  $U_t$ of $v_\pi (S_t)$ is used, which we call the \emph{update target}, resulting in the following general update rule:
\begin{equation}
\w_{t+1}  =\w_t + \alpha \big( U_t -  \w^\tr\x_t \big)\x_t\,.
\label{eq:linear update}
\end{equation}

There are many different update targets possible. For an unbiased estimator the full return can be used, that is,  $U_t = G_t$. However, the full return has the disadvantage that its variance is typically very high. Hence, learning with the full return can be slow. Temporal-difference (TD) learning addresses this issue by using update targets based on other value estimates. While the update target is no longer unbiased in this case, the variance is typically much smaller, and learning much faster. TD learning uses the Bellman equations as its mathematical foundation for constructing update targets. These equations relate the value of a state to the values of its successor states:
$$v_\pi(s) =  \sum_{a} \pi(s,a)  \sum_{s'} p(s' | s,a) \big(r(s,a,s') + \gamma  \, v_\pi(s') \big)\,.$$
Writing this equation in terms of an expectation yields:
$$v_\pi(s) =  \mathbb{E}\{ R_{t+1} + \gamma v_\pi(S_{t+1}) | S_t = s\}_{\pi, p, r}\,.$$
Sampling from this expectation, while using linear function approximation  to approximate $v_\pi$, results in the update target:
$$U_t = R_{t+1} + \gamma \,\w^\tr \x_{t+1}\,.$$
This update target is called a one-step update target, because it is based on information from only one time step ahead. Applying the Bellman equation multiple times results in update targets based on information further ahead. Such update targets are called multi-step update targets.

\subsection{TD($\lambda$)}

The TD($\lambda$) algorithm implements the following update equations:
\begin{eqnarray}
\delta_t &=& R_{t+1} + \gamma \w_t^\tr \x_{t+1}   - \w_{t}^\tr \x_{t} \,,\label{eq:delta_update}\\
{\bs e}_t &=& \gamma\lambda {\bs e}_{t-1} +  \x_t  \,,\label{eq:trace_update}\\
\w_{t+1} &=&  \w_t + \alpha \delta_t\,{\bs e}_{t}  \,,\label{eq:weight_update}
\end{eqnarray}
for $t \geq 0$, and with ${\bs e}_{-1} = {\bs 0}$. The scalar $\delta_t$ is called the \emph{TD error}, and the vector ${\bs e}_t$ is called the \emph{eligibility-trace} vector.  The update of ${\bs e}_{t}$ shown above is referred to as the \emph{accumulating-trace} update. As a shorthand, we will refer to this version of TD($\lambda$) as `accumulate TD($\lambda$)', to distinguish it from a slightly different version that is discussed below. 
While these updates appear to deviate from the general, gradient-descent-based update rule given in (\ref{eq:linear update}), there is a close connection to this update rule. 
This connection is formalized through the forward view of TD($\lambda$), which we discuss in detail in the next section. Algorithm \ref{al:TD(lambda)} shows the pseudocode for accumulate TD($\lambda$).
\begin{algorithm}[thb]
\begin{algorithmic}[0]
\STATE {\bf INPUT: $\alpha, \lambda, \gamma, \w_{init}$}
\STATE $\w \leftarrow \w_{init}$
\STATE Loop (over episodes):
\STATE \qquad obtain initial $\x$
\STATE \qquad${\bs e} \leftarrow {\bs 0}$
\STATE \qquad While terminal state has not been reached, do:
\STATE \qquad\qquad obtain next feature vector $\x'$ and reward $R$
\STATE \qquad\qquad $\delta \leftarrow R + \gamma\, \w^\tr\x' -  \w^\tr\x$
\STATE \qquad\qquad $ {\bs e} \leftarrow  \gamma\lambda {\bs e}  + \x$
\STATE \qquad\qquad $\w \leftarrow  \w + \alpha \delta  {\bs e}$
\STATE \qquad\qquad $\x \leftarrow \x' $
\caption{accumulate TD($\lambda$)}
\label{al:TD(lambda)}
\end{algorithmic}
\end{algorithm}

Accumulate TD($\lambda$) can be very sensitive with respect to the $\alpha$ and $\lambda$ parameters. Especially, a large value of $\lambda$ combined with a large value of $\alpha$ can easily cause divergence, even on simple tasks with bounded rewards. For this reason, a variant of TD($\lambda$) is sometimes used that is more robust with respect to these parameters. This variant, which assumes binary features, uses a different trace-update equation:
 \begin{displaymath}
{\bs e}_{t}[i]= \begin{cases} \gamma\lambda  {\bs e}_{t-1}[i]\,,  &\mbox{if } \x_{t}[i] = 0;\\
1\,, & \mbox{if }  \x_{t}[i]= 1\,, \end{cases} 
\qquad\mbox{ for all features } i\thinspace.
\end{displaymath}
where ${\bs x}[i]$ indicates the $i$-th component of vector ${\bs x}$. This update is referred to as the \emph{replacing-trace} update. As a shorthand, we will refer to the version of TD($\lambda$) using the replacing-trace update as `replace TD($\lambda$)'.

\section{The Online Forward View}
\label{sec:online forward view}

The traditional forward view relates the TD($\lambda$) update equations to the general update rule shown in Equation (\ref{eq:linear update}). Specifically, for small step-sizes the weight vector at the end of an episode computed by accumulate TD($\lambda$) is approximately the same as the weight vector resulting from a sequence of Equation (\ref{eq:linear update}) updates (one for each visited state) using a particular multi-step update target, called the \emph{$\lambda$-return} \citep{sutton:book98, dimitri:book96}.  The $\lambda$-return for state $S_t$ is defined as:
\begin{equation}
G^{\lambda}_t := (1-\lambda) \sum_{n=1}^{T-t-1}  \lambda^{n-1} G_t^{(n)}  +  \lambda^{T-t-1} G_t\,,
\label{eq:interim lambda return1}
\end{equation}
where $T$ is the time step the terminal state is reached, and $G_t^{(n)}$ is the $n$-step return, defined as:
\begin{displaymath}
G_t^{\,(n)} := \sum_{k=1}^n \gamma^{k-1} R_{t+k} + \gamma^n\, V(S_{t+n}| \w_{t+n-1}).
\end{displaymath}

We call a method that updates the value of each visited state at the end of the episode an \emph{offline} method; we call a method that updates the value of each visited state immediately after the visit (i.e., at the time step after the visit) an \emph{online} method. TD($\lambda$) is an online method. The update sequence of the traditional forward view, however, corresponds with an offline method, because the $\lambda$-return requires data up to the end of an episode. This leaves open the question of how to interpret the weights of TD($\lambda$) \emph{during} an episode. In this section, we provide an answer to this long-standing open question. We introduce a bounded version of the $\lambda$-return that only uses information up to a certain horizon and we use this to construct an online forward view. This online forward view approximates the weight vectors of accumulate TD($\lambda$) at \emph{all} time steps, instead of only at the end of an episode.

\subsection{The Online $\lambda$-Return Algorithm}
\label{sec:online forward view equations}

The concept of an online forward view contains a paradox. On the one hand, multi-step update targets require data from time steps far beyond the time a state is visited; on the other hand, the online aspect requires that the value of a visited state is updated immediately. The solution to this paradox is to assign a sequence of update targets to each visited state. The first update target in this sequence contains data from only the next time step, the second contains data from the next two time steps, the third from the next three time steps, and so on. Now, given an initial weight vector and a sequence of visited states, a new weight vector can be constructed by updating each visited state with an update target that contains data up to the current time step. Below, we formalize this idea.

We define the \emph{interim $\lambda$-return} for state $S_k$ with horizon $h \in \mathbb{N}^+, h > k$ as follows:
\begin{equation}
G^{\lambda|h}_k := (1-\lambda) \sum_{n=1}^{h-k-1}  \lambda^{n-1} G_k^{(n)} + \lambda^{h-k-1} G_k^{(h-k)}\,.
\label{eq:interim lambda return1}
\end{equation}
Note that this update target does not use data beyond the horizon $h$. $G^{\lambda|h}_k$ implicitly defines a sequence of update targets for $S_k$:  $\{ G_k^{\lambda|k+1}, G_k^{\lambda|k+2}, G_k^{\lambda|k+3}, \dots \}$. As time increases, update targets based on data further away become available for state $S_k$. At a particular time step $t$, a new weight vector is computed by performing an Equation (\ref{eq:linear update}) update for each visited state using the interim $\lambda$-return with horizon $t$, starting from the initial weight vector $\w_{init}$. 
Hence, at time step $t$, a sequence of $t$ updates occurs. To describe this sequence mathematically, we use weight vectors with two indices: $\w_k^{\,t}$. The superscript indicates the time step at which the updates are performed (this value corresponds with the horizon of the interim $\lambda$-returns that are used in the updates).  The subscript is the iteration index of the sequence (it corresponds with the number of updates that have been performed at a particular time step). 
As an example, the update sequences for the first three time steps are:
\begin{eqnarray*}
&t = 1 : & \w_{1}^1 = \w_0^1 + \alpha \big(G_0^{\lambda|1}  - (\w_0^1)^\tr\, \x_0\big)\x_0\thinspace,\\
&&\\
&t = 2 : & \w_1^2 = \w_0^2 + \alpha \big( G_0^{\lambda|2}  -  (\w_0^2)^\tr \,\x_0\big)\x_0 \thinspace,\\
&&\w_2^2 = \w_1^2 + \alpha \big( G_1^{\lambda|2}  - (\w_1^2)^\tr \,\x_1\big)\x_1\thinspace,\\
&&\\
&t = 3 : & \w_1^3 = \w_0^3 + \alpha \big( G_0^{\lambda|3}  - (\w_0^3)^\tr \,\x_0\big)\x_0 \thinspace,\\
&&\w_2^3 = \w_1^3 + \alpha \big( G_1^{\lambda|3} - (\w_1^3)^\tr \,\x_1\big)\x_1\thinspace,\\
&&\w_3^3 = \w_2^3 + \alpha \big( G_2^{\lambda|3} - (\w_2^3)^\tr \,\x_2\big)\x_2\thinspace,
\end{eqnarray*}
with $\w_0^t := \w_{init}$ for all $t$.  
%Note that we use two indices for the weight vectors. The subscript indicates how many updates have been performed from the start of the sequence; the superscript indicates the horizon of the interim $\lambda$-returns that are used in the updates. 
More generally, the update sequence at time step  $t$ is:
\begin{equation}
\w_{k+1}^t := \w_k^t + \alpha \Big(G_k^{\lambda|t} - (\w_k^t)^\tr \,\x_k\Big)\x_k\thinspace, \qquad\mbox{ for } 0 \leq k < t\thinspace. \label{eq:real-time update}
\end{equation}
We define $\w_t$ (without superscript) as the final weight vector of the update sequence at time $t$, that is, $\w_t := \w_t^t$. We call the algorithm implementing Equation (\ref{eq:real-time update}) the \emph{online $\lambda$-return algorithm}. By contrast, we call the algorithm that implements the traditional forward view the \emph{offline $\lambda$-return algorithm}.

The update sequence performed by the online $\lambda$-return algorithm at time step T (the time step that a terminal state is reached) is very similar to the update sequence performed by the offline $\lambda$-return algorithm. In particular, note that  $G_t^{\lambda|T}$ and  $G_t^{\lambda}$ are the same, under the assumption that the weights used for the value estimates are the same. Because these weights are in practise not exactly the same, there will typically be a small difference.\footnote{If $\lambda = 1$ there is never a difference because there is no bootstrapping.}

Figure \ref{fig:offline vs online} illustrates the difference between the online and offline $\lambda$-return algorithm, as well as accumulate TD($\lambda$), by showing the RMS error on a random walk task. The task consists of 10 states laid out in a row plus a terminal state on the left. Each state transitions with 70\% probability to its left neighbour and with 30\% probability to its right neighbour (or to itself in case of the right-most state). All rewards are 1 and $\gamma = 1$. Furthermore, $\lambda = 1$ and $\alpha = 0.2$. The right-most state is the initial state. Whereas the offline $\lambda$-return algorithm only makes updates at the end of an episode, the online $\lambda$-return algorithm, as well as accumulate TD$(\lambda$), make updates at every time step.

\begin{figure}[tb]
\begin{center}
    \includegraphics[height=5cm]{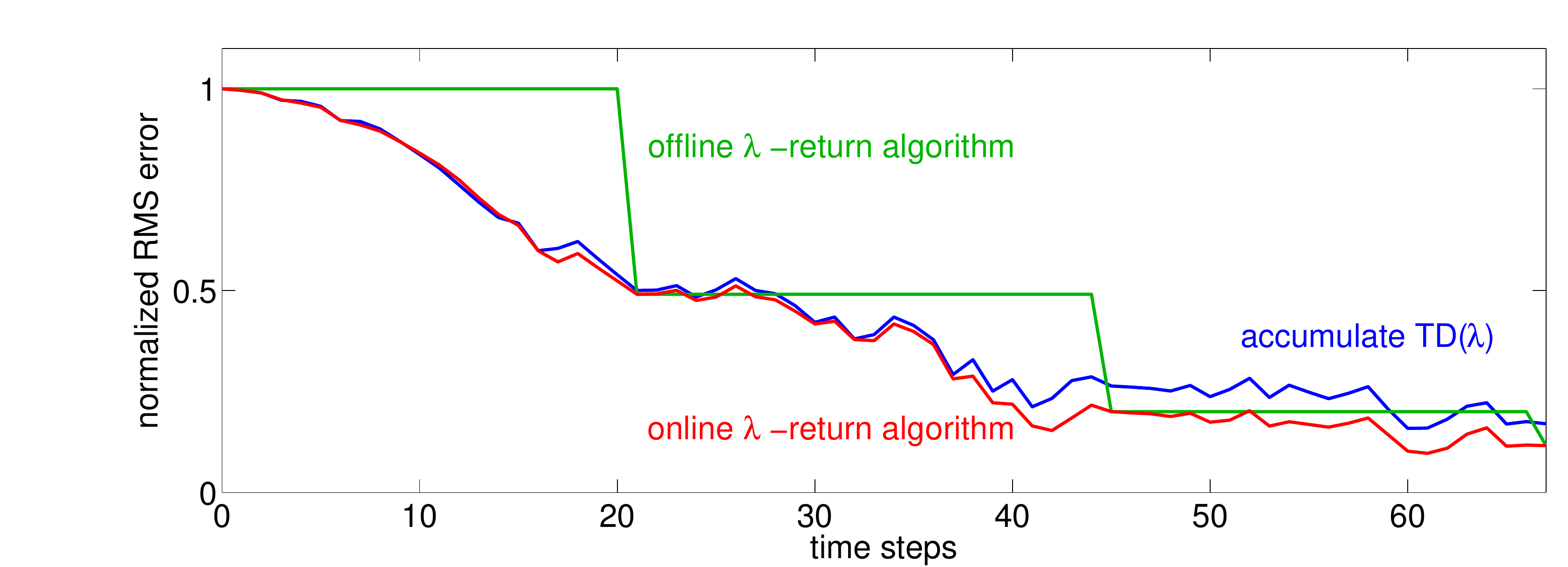}
    \caption{RMS error as function of time for the first 3 episodes of a random walk task, for $\lambda = 1$ and $\alpha = 0.2$. The error shown is the RMS error over all states, normalized by the initial RMS error.} 
    \label{fig:offline vs online}
\end{center}
\end{figure}

The comparison on the random walk task shows that accumulate TD($\lambda$) behaves similar to the online $\lambda$-return algorithm. In fact, the smaller the step-size, the smaller the difference between accumulate TD($\lambda$) and the online $\lambda$-return algorithm. This is formalized by Theorem 1. The proof of the theorem can be found in Appendix \ref{sec:proof}. The theorem uses the term $\Delta_i^t$, which is defined as:
$$\Delta_i^t := \big(\bar G_i^{\lambda|t} - \w_0^\tr \x_i \big) \x_i\,,$$ 
with $\bar G_i^{\lambda|t}$ the interim $\lambda$-return for state $S_i$ with  horizon $t$ that uses $\w_0$ for all value evaluations. Note that $\Delta_i^t$ is independent of the step-size.
\begin{theorem}
Let $\w_0$ be the initial weight vector, $\w_{t}^{td}$ be the weight vector at time $t$ computed by accumulate TD($\lambda$), and $\w_{t}^{\lambda}$ be the weight vector at time $t$ computed by the online $\lambda$-return algorithm.
Furthermore, assume that $\sum_{i=0}^{t-1} \Delta_i^t \neq {\boldsymbol 0}$. Then, for all time steps $t$:
$$\frac{|| \w_{t}^{td}  - \w_{t}^{\lambda}  ||}{|| \w_{t}^{td}  - \w_0  ||} \rightarrow 0\,, \qquad\mbox{as $\,\,\,\alpha \rightarrow 0$}.$$
\label{theorem:main}
\end{theorem}

Theorem \ref{theorem:main} generalizes the traditional result to arbitrary time steps. The traditional result  states that the difference between the weight vector at the end of an episode computed by the offline $\lambda$-return algorithm and the weight vector at the end of an episode computed by accumulate TD($\lambda$) goes to 0, if the step-size goes to 0 \citep{dimitri:book96}.

\subsection{Comparison to Accumulate TD($\lambda$)}

While accumulate TD($\lambda$) behaves like the online $\lambda$-return algorithm for small step-sizes, small step-sizes often result in slow learning. Hence, higher step-sizes are desirable. For higher step-sizes, however, the behaviour of accumulate TD($\lambda$) can be very different from that of the online $\lambda$-return algorithm. And as we show in the empirical section of this article (Section \ref{sec:empirical study}), when there is a difference, it is almost exclusively in favour of the online $\lambda$-return algorithm. In this section, we  analyze why the online $\lambda$-return algorithm can outperform accumulate TD($\lambda$), using the one-state example shown in the left of Figure \ref{fig:one-state example}.

\begin{figure}[tb]
\begin{center}
\includegraphics[width=4cm]{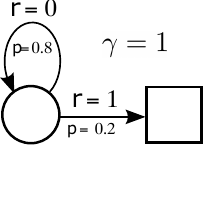}
\hspace{2cm}
\includegraphics[width=6cm]{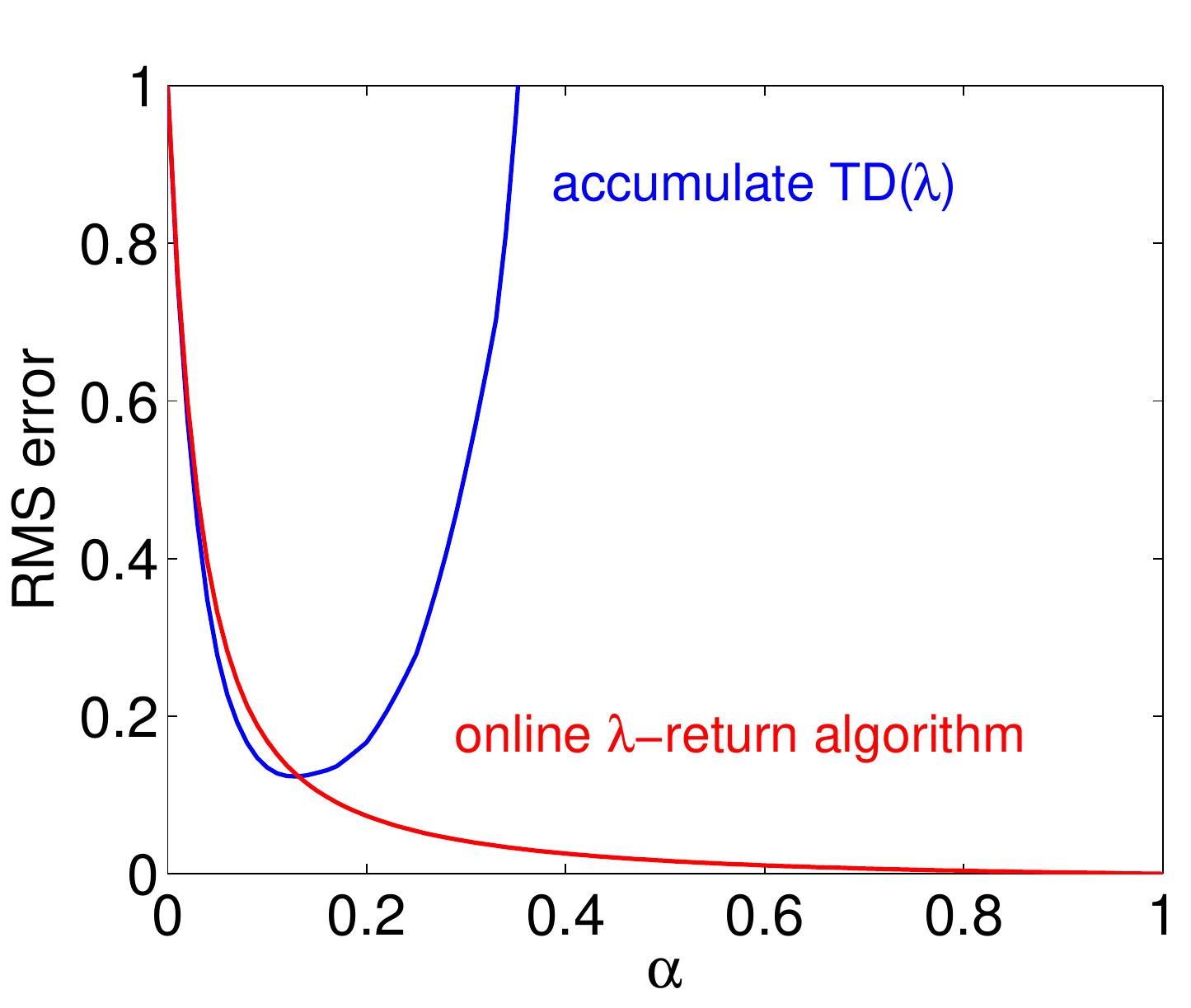}
\caption{{\it Left: } One-state example (the square indicates a terminal state). {\it Right: } The RMS error of the state value at the end of an episode, averaged over the first 10 episodes, for $\lambda = 1$.} 
\label{fig:one-state example}
\end{center}
\end{figure}

The right of Figure \ref{fig:one-state example} shows the RMS error over the first 10 episodes of the one-state example for different step-sizes and $\lambda=1$. While for small step-sizes accumulate TD($\lambda$) behaves indeed like the online $\lambda$-return algorithm---as predicted by Theorem \ref{theorem:main}---, for larger step-sizes the difference becomes huge. To understand the reason for this, we derive an analytical expression for the value at the end of an episode. 

First, we consider accumulate TD($\lambda$). Because there is only one state involved, we indicate the value of this state simply by V. The update at the end of an episode is  $V_T = V_{T-1} + \alpha e_{T-1} \delta_{T-1}$. In our example, $\delta_t = 0$ for all time steps  $t$, except for $t = {T-1}$, where  $\delta_{T-1} = 1 - V_{T-1}$. Because $\delta_t$  is 0 for all time steps except the last, $V_{T-1} = V_0$.  Furthermore, $\phi_t = 1$ for all time steps $t$, resulting in $e_{T-1}= T$. Substituting all this in the expression for $V_T$ yields:
\begin{equation}
V_T  = V_0 + T \alpha  (1  - V_0)\,,   \qquad\mbox{ for accumulate TD($\lambda$).}
\label{eq:VT acc}
\end{equation}
So for accumulate TD($\lambda$), the total value difference is simply a summation of the value difference corresponding to a single update.

%The analysis for replace TD($\lambda$) is very similar to the analysis of accumulate TD($\lambda$), except that $e_{T-1}$ is defined differently. Because replacing traces replaces the previous trace value after each revisit of a state, $e_{T-1}$ is only 1. Hence, the expression for $V_T$ becomes simply
%\begin{equation}
%V_T  = V_0 + \alpha  ( X  - V_0)\,,    \qquad\mbox{ for replace TD($\lambda$).}
%\label{eq:VT}
%\end{equation}
%So for replace TD($\lambda$) the total value difference is equal of the value difference corresponding to a single update.

Now, consider the online $\lambda$-return algorithm. The value at the end of an episode, $V_T$, is equal to $V_T^T$, resulting from the update sequence:
$$ V_{k+1}^T = V^T_k + \alpha (G_k^{\lambda|T} - V^T_k)\,, \qquad \mbox{ for } 0 \leq k < T\,.$$
By incremental substitution, we can directly express $V_T$ in terms of the initial value, $V_0$, and the update targets:
$$ V_T = (1-\alpha)^T V_0 + \alpha (1-\alpha)^{T-1} G_0^{\lambda|T}  + \alpha (1-\alpha)^{T-2} G_1^{\lambda|T}  + \cdots + \alpha G_{T-1}^{\lambda|T}\,.$$
Because $G_k^{\lambda|T} = 1$ for all $k$ in our example, the weights of all update targets can be added together and the expression can be rewritten as a single pseudo-update, yielding:
\begin{equation}
V_T  = V_0 + \big( 1 - (1-\alpha)^T\big) \cdot ( 1 - V_0)\,,  \qquad\mbox{ for the online $\lambda$-return algorithm.}
\label{eq:VT on}
\end{equation}

The term $1 - (1-\alpha)^T$ in (\ref{eq:VT on}) acts like a pseudo step-size. For larger $\alpha$ or $T$ this pseudo step-size increases in value, but as long as $\alpha \leq 1$ the value will never exceed 1. By contrast, for accumulate TD($\lambda$) the pseudo step-size is $T \alpha$, which can grow much larger than 1 even for $\alpha < 1$, causing divergence of values. This is the reason that accumulate TD($\lambda$) can be very sensitive to the step-size and it explains why the optimal step-size for accumulate TD($\lambda$) is much smaller than the optimal step-size for the online $\lambda$-return algorithm in Figure \ref{fig:one-state example} ($\alpha \approx 0.15$  versus $\alpha = 1$, respectively). Moreover, because the variance on the pseudo step-size is higher for accumulate TD($\lambda$) the performance at the optimal step-size for accumulate TD($\lambda$) is worse than the performance at the optimal step-size for the online $\lambda$-return algorithm. 

\subsection{Comparison to Replace TD($\lambda$)}

The sensitivity of accumulate TD($\lambda$) to divergence, demonstrated in the previous subsection, has been known for long. In fact, replace TD($\lambda$) was designed to deal with this. But while replace TD($\lambda$) is much more robust with respect to divergence, it also has its limitations. One obvious limitation is that it only applies to binary features, so it is not generally applicable. But even in domains where replace TD($\lambda$)  can be applied,  it  can perform poorly.  The reason is that replacing previous trace values, rather than adding to it, reduces the multi-step characteristics of TD($\lambda$). 

To illustrate this, consider the two-state example shown in the left of Figure \ref{fig:two-state example}. It is easy to see that the value of the left-most state is 2 and of the other state is 0. The state representation consists of only a single, binary feature that is 1 in both states and 0 in the terminal state. Because there is only a single feature, the state values cannot be represented exactly. The weight that minimizes the mean squared error assigns a value of 1 to both states, resulting in an RMS error of 1. Now consider the graph shown in the right of Figure \ref{fig:two-state example}, which shows the asymptotic RMS error for different values of $\lambda$. The error for accumulate TD($\lambda$) converges to the least mean squares (LMS) error for $\lambda = 1$, as predicted by the theory \citep{dayan:ml92}. The online $\lambda$-return algorithm has the same convergence behaviour (due to Theorem 1). By contrast, replace TD($\lambda$) converges to the same value as TD(0) for any value of $\lambda$. The reason for this behaviour is that because the single feature is active at all time steps, the multi-step behaviour of TD($\lambda$) is fully removed, no matter the value of $\lambda$. Hence, replace TD($\lambda$) behaves exactly the same as TD(0) for any value of $\lambda$ at all time steps. As a result, it also behaves like TD(0) asymptotically. 

The two-state example very clearly demonstrates that there is a price payed by replace TD($\lambda$) to achieve robustness with respect to divergence: a reduction in multi-step behaviour. By contrast, the online $\lambda$-return algorithm, which is also robust to divergence, does not have this disadvantage. Of course, the two-state example, as well as the one-state example from the previous section, are extreme examples, merely meant to illustrate what can go wrong. But in practise, a domain will often have some characteristics of the one-state example and some of the two-state example, which negatively impacts the performance of both accumulate and replace TD($\lambda$).
\begin{figure}[tb]
\begin{center}
\includegraphics[width=7cm]{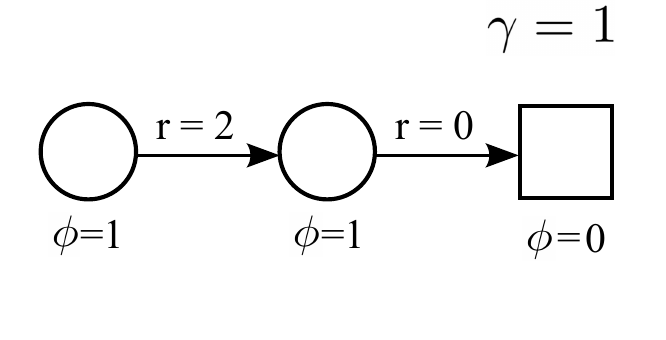}
\hspace{1cm}
\includegraphics[width=6cm]{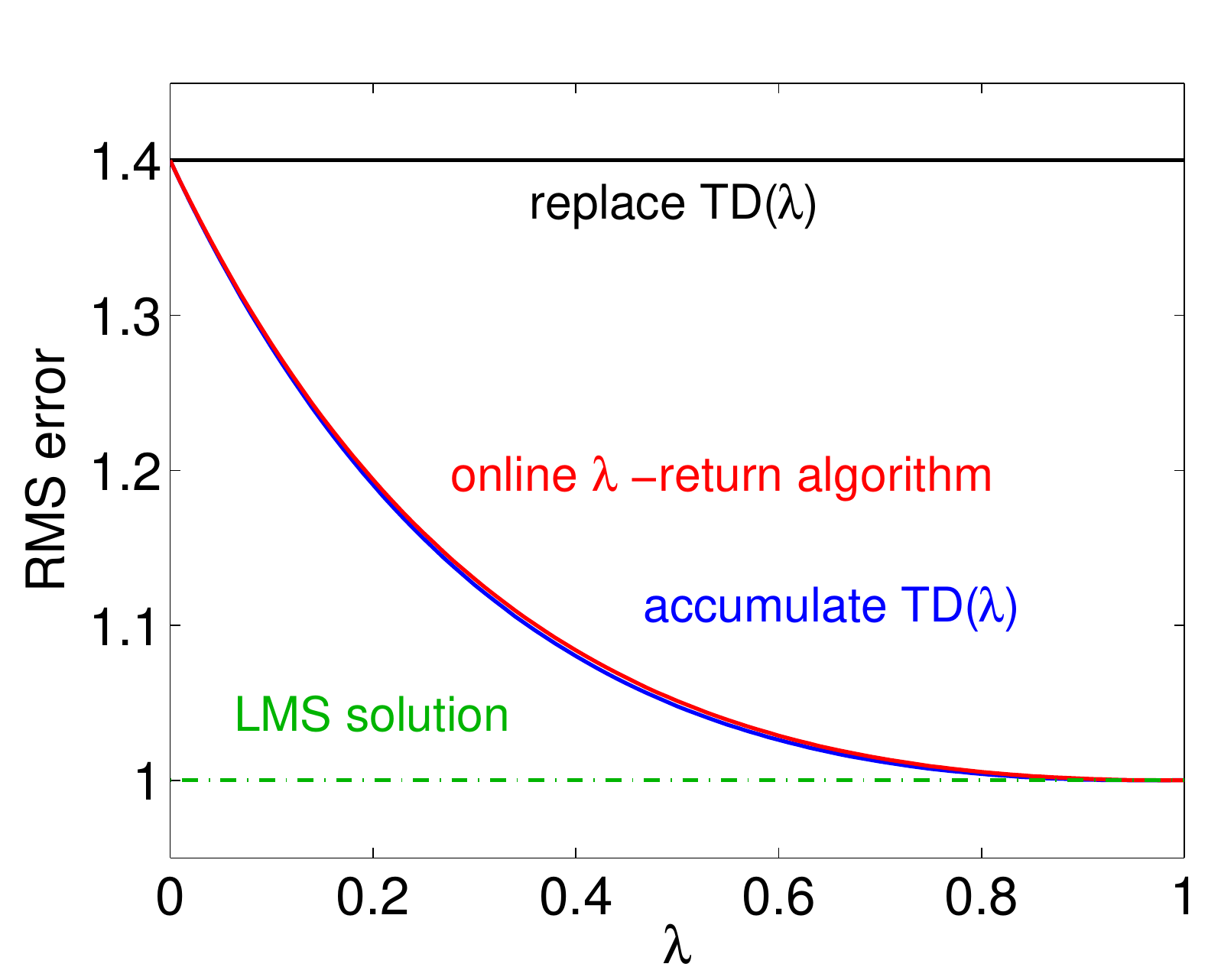}
\caption{{\it Left: } Two-state example. {\it Right: } The RMS error after convergence for different $\lambda$ (at $\alpha = 0.01$). We consider values to be converged if the error changed less than 1\% over the last 100 time steps. } 
\label{fig:two-state example}
\end{center}
\end{figure}

\section{True Online TD($\lambda$)}

The online $\lambda$-return algorithm is impractical on many domains: the memory it uses, as well as the computation required per time step increases linearly with time. Fortunately, it is possible to rewrite the update equations of the online $\lambda$-return algorithm  to a different set of update equations that can be implemented with a computational complexity that is independent of time. In fact, this alternative set of update equations differs from the update equations of accumulate TD($\lambda$) only by two extra terms, each of which can be computed efficiently. The algorithm implementing these equations is called true online TD($\lambda$) and is discussed below.

\subsection{The Algorithm}
\label{sec:algorithm}

For the online $\lambda$-return algorithm, at each time step a sequence of updates is performed. The length of this sequence, and hence the computation per time step, increases over time. However, it is possible to compute the weight vector resulting from the sequence at time step $t+1$ directly from the weight vector resulting from the sequence at time step $t$. This results in the following update equations (see Appendix \ref{sec:derivation} for the derivation):
\begin{eqnarray}
\delta_t &=& R_{t+1} + \gamma \w_t^\tr \x_{t+1}   - \w_{t}^\tr \x_{t} \,,\label{eq:to_delta_update}\\
{\bs e}_t &=& \gamma\lambda {\bs e}_{t-1} +  \x_t -  \alpha \gamma \lambda ({\bs e}_{t-1}^\tr \,\x_t) \,\x_t \,,\label{eq:to_trace_update}\\
\w_{t+1} &=&  \w_t + \alpha \delta_t\,{\bs e}_{t} + \alpha ( \w_{t}^\tr \x_{t} - \w_{t - 1}^\tr \x_{t} ) ( {\bs e}_t -  \x_t )\,,\quad \label{eq:to_weight_update}
\end{eqnarray}
for $t \geq 0$, and with ${\bs e}_{-1} = {\bs 0}$. Compared to accumulate TD($\lambda$), both the trace update and the weight update have an additional term. 
We call a trace updated in this way a \emph{dutch trace};  we call the term $\alpha ( \w_{t}^\tr \x_{t} - \w_{t - 1}^\tr \x_{t} ) ( {\bs e}_t -  \x_t )$ the \emph{TD-error time-step correction}, or simply the $\delta$-correction.
Algorithm \ref{al:true TD(lambda)} shows pseudocode that implements these equations.\footnote{When using a time-dependent step-size (e.g., when annealing the step-size) use the pseudocode from Section \ref{sec:time-dependent step-size}. For reasons explained in that section this requires a modified trace update. That pseudocode is the same as the pseudocode from \citet{vanseijen:icml14}.}

%\footnote{We provide pseudocode for true online TD($\lambda$) with time-dependent step-size in Section. For reasons explained in that section, this requires a modified trace update. In addition, for reference purposes, we provide pseudocode for the special case of tabular features in Section \ref{sec:tabular true online TD(lambda)}.} 
%(\ref{eq:to_delta_update}), (\ref{eq:to_trace_update}) and (\ref{eq:to_weight_update}).

\begin{algorithm}[tb]
\begin{algorithmic}[0]
\STATE {\bf INPUT: $\alpha, \lambda, \gamma, \w_{init}$}
\STATE $\w \leftarrow \w_{init}$
\STATE Loop (over episodes):
\STATE \qquad obtain initial $\x$
\STATE \qquad${\bs e} \leftarrow {\bs 0}  \,; \,\, { V}_{old} \leftarrow 0$
\STATE \qquad While terminal state has not been reached, do:
\STATE \qquad\qquad obtain next feature vector $\x'$ and reward $R$
\STATE \qquad\qquad $  V \leftarrow \w^\tr\x$
\STATE \qquad\qquad $ V' \leftarrow \w^\tr\x'$
\STATE \qquad\qquad $\delta \leftarrow R + \gamma\,  V' -  V$
\STATE \qquad\qquad $ {\bs e} \leftarrow  \gamma\lambda {\bs e}  + \x - \alpha \gamma\lambda  ({\bs e}^\tr \x )\,\x$
\STATE \qquad\qquad $\w \leftarrow  \w + \alpha (\delta +  V -  V_{old}) {\bs e} - \alpha ( V -  V_{old})\x$
\STATE  \qquad\qquad $ V_{old} \leftarrow  V'$
\STATE \qquad\qquad $\x \leftarrow \x' $
\caption{true online TD($\lambda$)}
\label{al:true TD(lambda)}
\end{algorithmic}
\end{algorithm}

In terms of computation time, true online TD($\lambda$) has a (slightly) higher cost due to the two extra terms that have to be accounted for. While the computation-time complexity of true online TD($\lambda$) is the same as that of accumulate/replace TD($\lambda$)---$\mathcal{O}(n)$ per time step with $n$ being the number of features---, the actual computation time can be close to twice as much in some cases. In other cases (for example if sparse feature vectors are used), the computation time  of true online TD($\lambda$) is only a fraction more than that of accumulate/replace TD($\lambda$). In terms of memory, true online TD($\lambda$) has the same cost as accumulate/replace TD($\lambda$).

\subsection{When Can a Performance Difference be Expected?}
\label{sec:relative performance}

In Section \ref{sec:online forward view}, a number of examples were shown where the online $\lambda$-return algorithm outperforms accumulate/replace TD($\lambda$). Because true online TD($\lambda$) is simply an efficient implementation of the online $\lambda$-return algorithm, true online TD($\lambda$) will outperform accumulate/replace TD($\lambda$) on these examples as well.  But not in all cases will there be a performance difference. For example, it follows from Theorem \ref{theorem:main} that when appropriately small step-sizes are used, the difference between the online $\lambda$-return algorithm/true online TD($\lambda$) and accumulate TD($\lambda$) is negligible. In this section, we identify two other factors that affect whether or not there will be a performance difference. While the focus of this section is on performance \emph{difference} rather than performance \emph{advantage}, our experiments will show that true online TD($\lambda$) performs always at least as well as accumulate TD($\lambda$) and replace TD($\lambda$). In other words, our experiments suggest that whenever there is a performance difference, it is in favour of true online TD($\lambda$).

The first factor is the $\lambda$ parameter and follows straightforwardly from the true online TD($\lambda$) update equations.
\begin{prop}
For $\lambda = 0$,  accumulate TD($\lambda$), replace TD($\lambda$) and the online $\lambda$-return algorithm / true online TD($\lambda$) behave the same.
\end{prop}
\begin{proof}
For $\lambda = 0$, the accumulating-trace update, the replacing-trace update and the dutch-trace update all reduce to ${\bs e}_t = \x_t$. In addition, because ${\bs e}_t = \x_t$, the $\delta$-correction of true online TD($\lambda$) is 0.
\end{proof}
Because the behaviour of TD($\lambda$) for small $\lambda$ is close to the behaviour of TD(0),  it follows that significant performance differences will only be observed when $\lambda$ is large.

The second factor is related to how often a feature has a non-zero value. We start again with a proposition that highlights a condition under which the different TD($\lambda$) versions behave the same. The proposition makes use of an accumulating trace at time step $t-1$, ${\bs e}^{acc}_{t-1}$, whose non-recursive form is:
\begin{equation}
{\bs e}^{acc}_{t-1} = \sum_{k=0}^{t-1} (\gamma\lambda)^{t-1-k}  \x_k \,.
\label{eq:non-recursive}
\end{equation}
Furthermore, the proposition uses ${\bs x}[i]$ to denote the $i$-th element of vector ${\bs x}$.
\begin{prop}
If for all features $i$ and at all time steps $t$
\begin{equation}
{\bs e}_{t-1}^{acc}[i] \cdot \x_t[i]  = 0\,,
\label{prop:eq}
\end{equation}
then accumulate TD($\lambda$), replace TD($\lambda$) and the online $\lambda$-return algorithm / true online TD($\lambda$) behave the same (for any $\lambda$).
\label{prop:e}
\end{prop}
\begin{proof}
Condition (\ref{prop:eq}) implies that if $\x_t[i] \neq 0$, then ${\bs e}_{t-1}^{acc}[i]  = 0$. From this it follows that for binary features the accumulating-trace update can be written as a replacing-trace update at every time step:
\begin{eqnarray*}
{\bs e}_t^{acc}[i] &:=& \gamma\lambda {\bs e}_{t-1}^{acc}[i] +  \x_t[i] \,,\\
&=&  \begin{cases} \gamma\lambda  {\bs e}^{acc}_{t-1}[i]\,,  &\mbox{if } \x_{t}[i] = 0;\\
1\,, & \mbox{if }  \x_{t}[i]= 1\,.\end{cases}
\end{eqnarray*}
Hence, accumulate TD($\lambda$) and replace TD($\lambda$) perform exactly the same updates. 

Furthermore, condition (\ref{prop:eq}) implies that $({\bs e}_{t-1}^{acc})^\tr \x_t = 0$. Hence, the accumulating-trace update can also be written  as a dutch trace update at every time step:
\begin{eqnarray*}
{\bs e}_t^{acc} &:=& \gamma\lambda {\bs e}_{t-1}^{acc} +  \x_t \,,\\
&=& \gamma\lambda {\bs e}_{t-1}^{acc} +  \x_t   -  \alpha \gamma \lambda (({\bs e}_{t-1}^{acc})^\tr \,\x_t) \,\x_t \,.
\end{eqnarray*}
In addition, note that the $\delta$-correction is proportional to $\w_{t}^\tr \x_{t} - \w_{t - 1}^\tr \x_{t}$, which can be written as $\big(\w_{t} - \w_{t - 1} \big)^\tr\x_{t}$. The value $\big(\w_{t} - \w_{t - 1} \big)^\tr\x_{t}$ is proportional to  $({\bs e}_{t-1}^{acc})^\tr \x_t$ for accumulate TD($\lambda$). Because $({\bs e}_{t-1}^{acc})^\tr \x_t = 0$, accumulate TD($\lambda$) can add a $\delta$-correction at every time step without any consequence. This shows that accumulate TD($\lambda$) makes the same updates as true online TD($\lambda$). 
\end{proof}
An example of a domain where the condition of Proposition \ref{prop:e} holds is a domain with tabular features (each state is represented with a unique standard-basis vector), where a state is never revisited within the same episode. 

The condition of Proposition \ref{prop:e} holds approximately when the value $\left| {\bs e}_{t-1}^{acc}[i] \cdot \x_t[i] \right|$ is close to 0 for all features at all time steps. In this case, the different TD($\lambda$) versions will perform very similarly. It follows from Equation (\ref{eq:non-recursive}) that this is the case when there is a long time delay between the time steps that a feature has a non-zero value. Specifically, if there is always at least $n$ time steps between two subsequent times that a feature $i$ has a non-zero value with $\gamma\lambda^n$ being very small, then $\left| {\bs e}_{t-1}^{acc}[i] \cdot \x_t[i] \right|$ will always be close to 0. Therefore, in order to see a large performance difference, the same features should have a non-zero value often and within a small time frame (relative to $\gamma\lambda$). %This is for example the case on domains with non-sparse feature vectors, for which at any particular time step most features have a non-zero value. 

Summarizing the analysis so far: in order to see a performance difference $\alpha$ and $\lambda$ should be sufficiently large, and the same features should have a non-zero value often and within a small time frame. Based on this summary, we can address a related question: on what type of domains will there be a performance difference between true online TD($\lambda$) with optimized parameters and accumulate/replace TD($\lambda$) with optimized parameters. The first two conditions suggest that the domain should result in a relatively large optimal $\alpha$ and optimal  $\lambda$. This is typically the case for domains with a relatively low variance on the return. The last condition can be satisfied in multiple ways. It is for example satisfied by domains that have non-sparse feature vectors (that is, domains for which at any particular time step most features have a non-zero value). 

\subsection{True Online Sarsa($\lambda$)}

 TD($\lambda$) and true online TD($\lambda$) are policy evaluation methods. However, they can be turned into control methods in a straightforward way.
From a learning perspective, the main difference is that the prediction of the expected return should be conditioned on the state and action, rather than only on the state. 
This means that an estimate of the action-value function $q_{\pi}$ is being learned, rather than of the state-value function $v_{\pi}$. 

%Hence, the prediction is a function of the action feature-vectors instead of state feature-vectors. Because the predictions are conditioned on the state and the action, an estimate of the action-value function $q_{\pi}$ is being learned, rather than of the state-value function $v_{\pi}$. 

Another difference is that instead of having a fixed policy that generates the behaviour, the policy depends on the action-value estimates. Because these estimates typically improve over time, so does the policy. The (on-policy) control counterpart of TD($\lambda$) is the popular Sarsa($\lambda$) algorithm. The control counterpart of true online TD($\lambda$) is `true online Sarsa($\lambda$)'. Algorithm \ref{al:true online Sarsa(lambda)} shows pseudocode for true online Sarsa($\lambda$).

\begin{algorithm}[tb]
\begin{algorithmic}[0]
\STATE {\bf INPUT: $\alpha, \lambda, \gamma, \w_{init} $}
\STATE $\w \leftarrow \w_{init}$
\STATE Loop (over episodes):
\STATE \qquad obtain initial state $S$
\STATE \qquad select action $A$ based on state $S$ \quad (for example $\epsilon$-greedy)
\STATE \qquad ${\bs \psi} \leftarrow $ features corresponding to $S, A$
\STATE \qquad ${\bs e} \leftarrow {\bs 0}\,;\,\, {Q}_{old} \leftarrow 0$
\STATE \qquad While terminal state has not been reached, do:
\STATE \qquad\qquad take action $A$, observe next state $S'$ and reward $R$
\STATE \qquad\qquad select action $A'$ based on state $S'$
\STATE \qquad\qquad ${\bs \psi}' \leftarrow $ features corresponding to $S', A'$   \quad(if $S'$ is terminal state, ${\bs \psi}'  \leftarrow {\bs 0}$)
\STATE \qquad\qquad $ Q \leftarrow \w^\tr {\bs \psi}$ 
\STATE \qquad\qquad $Q' \leftarrow \w^\tr {\bs \psi'} $
\STATE \qquad\qquad $\delta \leftarrow R + \gamma\, Q' - Q$
\STATE \qquad\qquad $ {\bs e} \leftarrow  \gamma\lambda{\bs e} + {\bs \psi}  - \alpha  \gamma\lambda ({\bs e}^\tr {\bs \psi}  )\,{\bs \psi} $
\STATE \qquad\qquad $\w \leftarrow  \w + \alpha (\delta + Q - Q_{old})\,{\bs e} - \alpha (Q - Q_{old}) {\bs \psi}$
\STATE  \qquad\qquad $Q_{old} \leftarrow Q'$
\STATE \qquad\qquad ${\bs \psi} \leftarrow {\bs \psi'}  \,;\,\,A \leftarrow A'$
\caption{true online Sarsa($\lambda$)}
\label{al:true online Sarsa(lambda)}
\end{algorithmic}
\end{algorithm}

To ensure accurate estimates for all state-action values are obtained, typically some exploration strategy has to be used. 
A simple, but often sufficient strategy is to use an $\epsilon$-greedy behaviour policy.  That is, given current state $S_t$, with probability $\epsilon$ a random action is selected, and with probability $1 - \epsilon$ the greedy action is selected:
$$ A_{t}^{greedy} = \arg\max_a  \w_t^\tr {\bs \psi}(S_t,a)\,,$$
with ${\bs \psi}(s,a)$ an action-feature vector, and $\w_t^\tr {\bs \psi}(s,a)$ a (linear) estimate of $q_\pi(s,a)$ at time step $t$. A common way to derive an action-feature vector ${\bs \psi}(s,a)$ from a state-feature vector $\x(s)$ involves an action-feature vector of size $n|\mathcal{A}|$, where $n$ is the number of state features and $|\mathcal{A}|$ is the number of actions. Each action corresponds with a block of $n$ features in this action-feature vector. The features in ${\bs \psi}(s,a)$ that correspond to action $a$ take on the values of the state features; the features corresponding to other actions have a value of 0.

\section{Empirical Study}
\label{sec:empirical study}

This section contains our main empirical study, comparing TD($\lambda$), as well as Sarsa($\lambda$), with their true online counterparts. For each method and each domain, a scan over the step-size $\alpha$ and the trace-decay parameter $\lambda$ is performed such that the optimal performance can be compared. In Section \ref{sec:discussion}, we discuss the results.

\subsection{Random MRPs}
\label{sec:random mrps}

For our first series of experiments we used randomly constructed Markov reward processes (MRPs).\footnote{The code for the MRP experiments is published online at: \url{https://github.com/armahmood/totd-rndmdp-experiments}. The process we used to construct the MRPs is based on the process used by Bhatnagar, Sutton, Ghavamzadeh and Lee (2009).} An MRP can be interpreted as an MDP with only a single action per state. Consequently, there is only one policy possible. We represent a random MRP as a 3-tuple $(k, b, \sigma)$, consisting of $k$, the number of states; $b$, the branching factor (that is, the number of next states with a non-zero transition probability); and $\sigma$, the standard deviation of the reward. 
An MRP is constructed as follows. The $b$ potential next states for a particular state are drawn from the total set of states at random, and without replacement. The transition probabilities to those states are randomized as well (by partitioning the unit interval at $b-1$ random cut points).
The expected value of the reward for a transition is drawn from a normal distribution with zero mean and unit variance. The actual reward is drawn from a normal distribution with a mean equal to this expected reward and standard deviation $\sigma$. There are no terminal states.

We compared the performance of TD($\lambda$) on three different MRPs: one with a small number of states, $(10, 3, 0.1)$, one with a larger number of states, $(100, 10, 0.1)$, and one with a larger number of states but a low branching factor and no stochasticity for the reward, $(100, 3, 0)$.  The discount factor $\gamma$ is $0.99$ for all three MRPs. Each MRP is evaluated using three different representations. The first representation consists of \emph{tabular} features, that is, each state is represented with a unique standard-basis vector of $k$ dimensions. The second representation is based on \emph{binary} features. This binary representation is constructed by first assigning indices, from 1 to $k$, to all states. Then, the binary encoding of the state index is used as a feature vector to represent that state. The length of a feature vector is determined by the total number of states: for $k = 10$, the length is 4; for $k = 100$, the length is 7. As an example, for $k = 10$ the binary feature vectors of states 1, 2 and 3 are $(0,0,0,1)$,$(0,0,1,0)$ and $(0,0,1,1)$, respectively. Finally, the third representation uses non-binary features. For this representation each state is mapped to a 5-dimensional feature vector, with the value of each feature drawn from a normal distribution with zero mean and unit variance. After all the feature values for a state are drawn, they are normalized such that the feature vector has unit length. Once generated, the feature vectors are kept fixed for each state. Note that replace TD($\lambda$) cannot be used with this representation, because replacing traces are only defined for binary features (tabular features are a special case of this).

In each experiment, we performed a scan over $\alpha$ and $\lambda$. Specifically, between 0 and 0.1, $\alpha$ is varied according to $10^i$ with $i$ varying from -3 to -1 with steps of 0.2, and from 0.1 to 2.0 (linearly) with steps of 0.1. In addition, $\lambda$ is varied from 0 to 0.9 with steps of 0.1 and from 0.9 to 1.0 with steps of 0.01. The initial weight vector is the zero vector in all domains. As performance metric we used the mean-squared error (MSE) with respect to the LMS solution during early learning (for k = 10, we averaged over the first 100 time steps; for k = 100, we averaged over the first 1000 time steps). We normalized this error  by dividing it by the MSE under the initial weight estimate. 

\begin{figure}[tb]
\begin{center}
\includegraphics[width=\columnwidth]{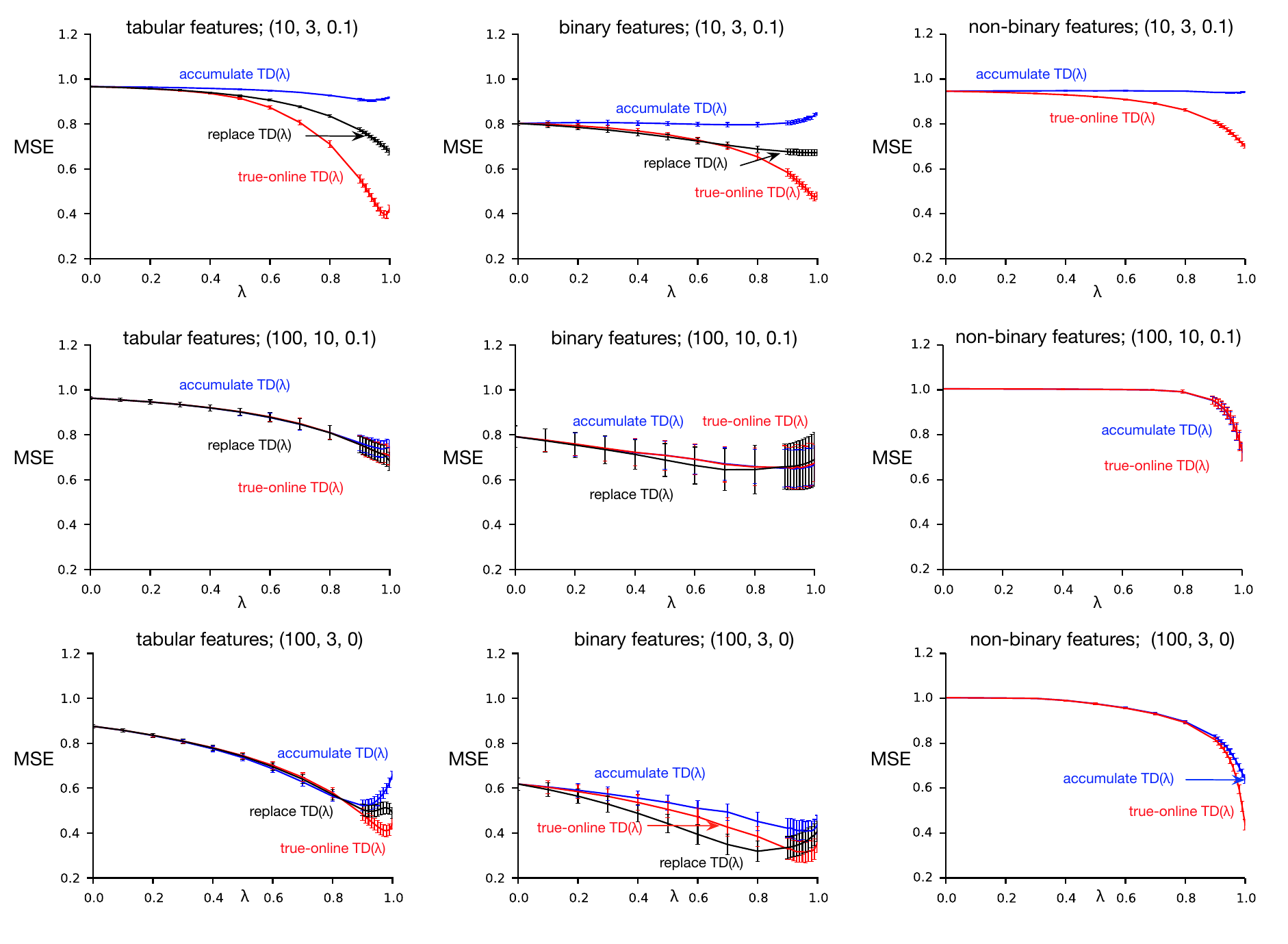}
\caption{MSE error during early learning for three different  MRPs, indicated by $(k, b, \sigma)$, and three different representations. The error shown is at optimal $\alpha$ value. }
\label{fig:mdp_all}
\end{center}
\end{figure}

Figure \ref{fig:mdp_all} shows the results for different $\lambda$ at the best value of $\alpha$.  In Appendix \ref{sec:detailed mrps}, the results for all $\alpha$ values are shown. The optimal performance of true online TD($\lambda$) is at least as good as the optimal performance of accumulate TD($\lambda$) and replace TD($\lambda$), on all domains and for all representations. A more in-depth discussion of these results is provided in Section  \ref{sec:discussion}.

\subsection{Predicting Signals From a Myoelectric Prosthetic Arm}

In this experiment, we compared the performance of true online TD($\lambda$) and TD($\lambda$) on a real-world data-set consisting of sensorimotor signals measured during the human control of an electromechanical robot arm. The source of the data is a series of manipulation tasks performed by a participant with an amputation, as presented by \cite{pilarski:ra13}. In this study, an amputee participant used signals recorded from the muscles of their residual limb to control a robot arm with multiple degrees-of-freedom (Figure \ref{fig:myocontrol}). Interactions of this kind are known as {\em myoelectric control} \citep[see, for example,][]{parker:jek06}. 

\begin{figure}[tbp]
\begin{center}
\includegraphics[width=4in]{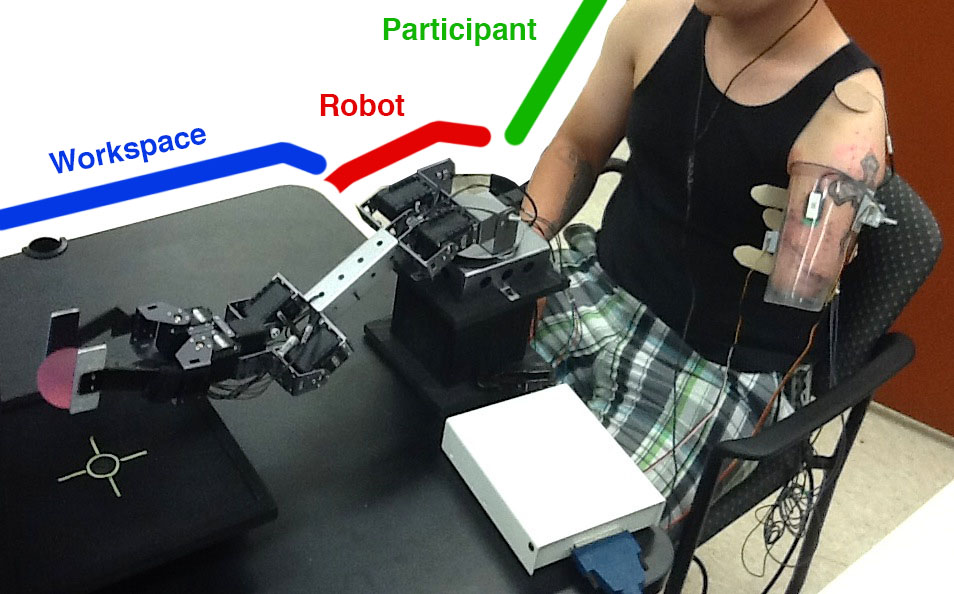}
\caption{Source of the input data stream and predicted signals used in this experiment: a participant with an amputation performing a simple grasping task using a myoelectrically controlled robot arm, as described in \cite{pilarski:ra13}. More detail on the subject and experimental setting can be found in \citet{hebert:nsre14}.}
\label{fig:myocontrol}
\end{center}
\end{figure}

For consistency and comparison of results, we used the same source data and prediction learning architecture as published in \cite{pilarski:ra13}. In total, two signals are predicted: grip force and motor angle signals from the robot's hand. 
Specifically, the target for the prediction is a discounted sum of each signal over time, similar to return predictions \citep[see general value functions and nexting;][]{sutton:aamas11, modayil:ab14}.
Where possible, we used the same implementation and code base as \cite{pilarski:ra13}. Data for this experiment consisted of 58,000 time steps of recorded sensorimotor information, sampled at 40 Hz (i.e., approximately 25 minutes of experimental data). The state space consisted of a tile-coded representation of the robot gripper's position, velocity, recorded gripping force, and two muscle contraction signals from the human user.  A standard implementation of tile-coding was used, with ten bins per signal, eight overlapping tilings, and a single active bias unit. This results in a state space with 800,001 features, 9 of which were active at any given time.  Hashing was used to reduce this space down to a vector of 200,000 features that are then presented to the learning system.  All signals were normalized between 0 and 1 before being provided to the function approximation routine. The discount factor for predictions of both force and angle was  $\gamma=0.97$, as in the results presented by \cite{pilarski:ra13}. Parameter sweeps over  $\lambda$ and $\alpha$ are conducted for all three methods.
The performance metric is the mean absolute return error over all 58,000 time steps of learning, normalized by dividing by the error for $\lambda = 0$. 

Figure  \ref{fig:myoelectricresults} shows the performance for the angle as well as the force predictions at the best $\alpha$ value for different values of $\lambda$. In Appendix \ref{sec:detailed myo}, the results for all $\alpha$ values are shown. The relative performance of replace TD($\lambda$) and accumulate TD($\lambda$) depends on the predictive question being asked. For predicting the robot's grip force signal---a signal with small magnitude and rapid changes---replace TD($\lambda$) is better  than accumulate TD($\lambda$) at all $\lambda$ values larger than 0. However, for predicting the robot's hand actuator position, a smoothly changing signal that varies between a range of $\sim$300--500, accumulate TD($\lambda$)  dominates replace TD($\lambda$). On both prediction tasks, true online TD($\lambda$) dominates accumulate TD($\lambda$) and replace TD($\lambda$).

\begin{figure}[tb]
\begin{center}
\includegraphics[width=4.5in]{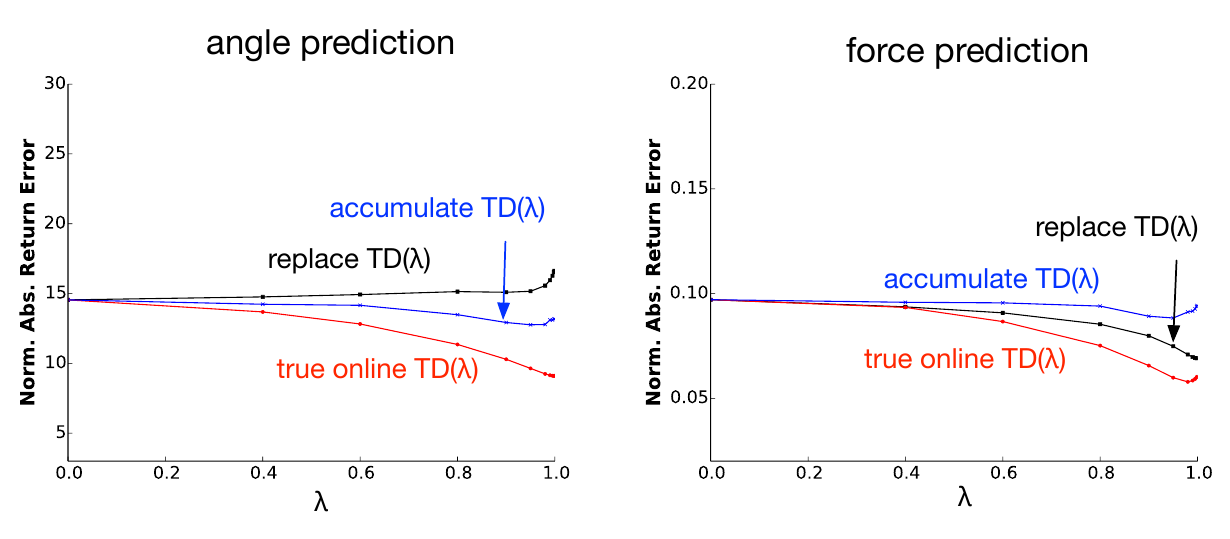}\\
\caption{Performance as function of $\lambda$ at the optimal $\alpha$ value, for the prediction of the servo motor angle (left), as well as the grip force (right).}
\label{fig:myoelectricresults}
\end{center}
\end{figure}

\subsection{Control in the ALE Domain Asterix}

In this final experiment, we compared the performance of true online Sarsa($\lambda$) with that of accumulate Sarsa($\lambda$) and replace Sarsa($\lambda$), on a domain from the Arcade Learning Environment (ALE) \citep{bellemare:jair13, defazio:arxiv14, mnih:nature15}, called Asterix. The ALE is a general testbed that 
provides an interface to hundreds of Atari 2600 games.\footnote{We used  ALE version 0.4.4 for our experiments. The code for the Asterix experiments is published online at: \url{https://github.com/mcmachado/TrueOnlineSarsa}.} 

%in which one has access,
%at each frame, to the game screen, the current RAM state and to a reward signal obtained from
%the transition between game frames. 
%At each frame the agent provides one of the 18 possible actions 
%in the game (equivalent to the 18 different actions allowed in the joystick) with the goal of maximizing the (discounted) sum of rewards. 

In the Asterix domain, the agent controls a yellow avatar, which has to collect `potion' objects, while avoiding `harp' objects  (see Figure \ref{fig:asterix} for a screenshot). Both potions and harps move across the screen horizontally. Every time the agent collects a potion it
receives a reward of 50 points, and every time it touches a harp it looses a life (it has three lives in total). 
The agent can use the actions \emph{up, right, down,} and \emph{left},  combinations of two directions, and a \emph{no-op} action, resulting in 9 actions in total.
The game ends after the agent has lost three lives, or after 5 minutes, whichever comes first.
%\footnote{We added the 5 minute time limit ourselves as in previous work \citep{bellemare:jair13}; the original game has no time limit.} 
\begin{figure}[tb]
  \centering
    \includegraphics[width=0.4\textwidth]{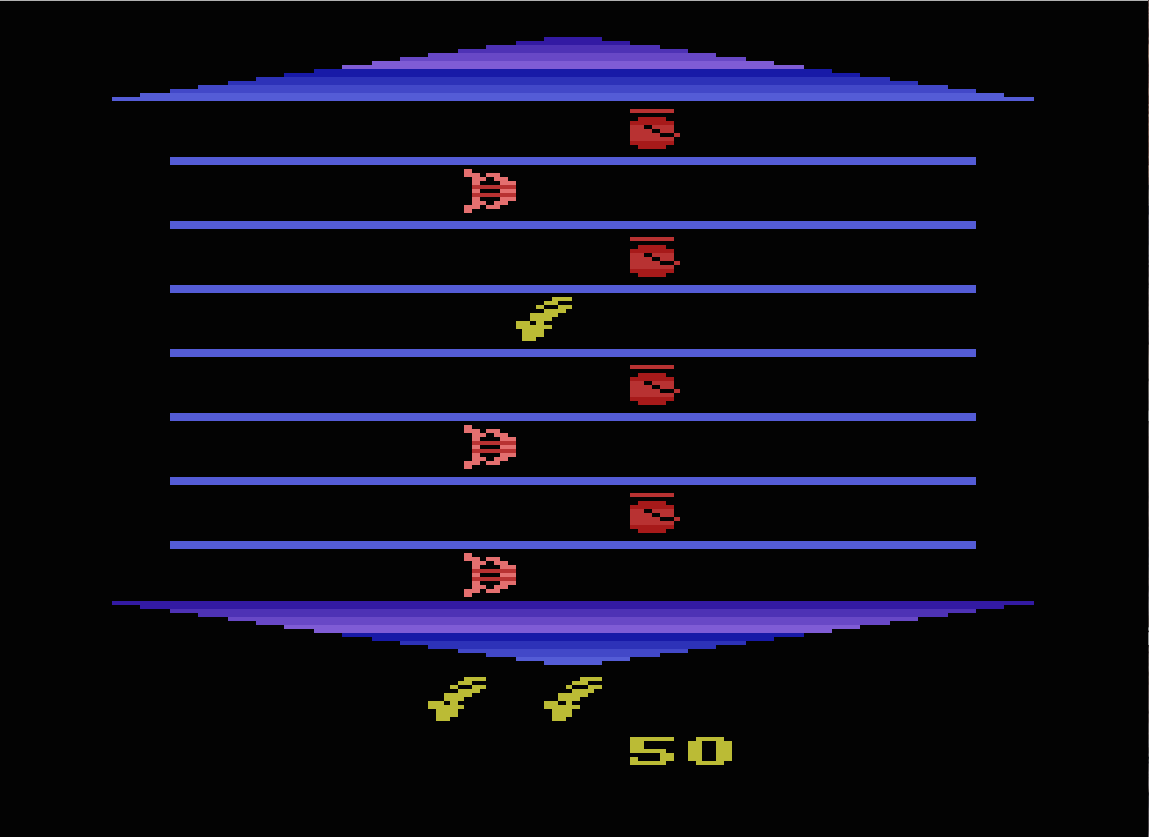}
    \caption{Screenshot of the game \textsc{Asterix}.}
  \label{fig:asterix}    
\end{figure}

We use linear function approximation using features derived from the screen pixels. Specifically, we use what \citet{bellemare:jair13} call the
\emph{Basic} feature set, which ``encodes the presence of colours on the Atari 2600 screen.'' 
It is obtained by first subtracting the game screen background \citep[see][sec.~3.1.1]{bellemare:jair13} 
and then dividing the remaining screen in to $16 \times 14$ tiles of size $10 \times 15$ pixels. 
Finally, for each tile, one binary feature is generated for each of the $128$ available colours, 
encoding whether a colour is active or not in that tile. This generates 28,672 features (plus a 
bias term).

Because episode lengths can vary hugely (from about 10 seconds all the way up to 5 minutes), constructing a fair performance metric is non-trivial. 
For example, comparing the average return on the first $N$ episodes of two methods is only fair if they have seen roughly the same amount of samples in those episodes, which is not guaranteed for this domain. On the other hand, looking at the total reward collected for the first $X$ samples is also not a good metric, because there is no negative reward associated to dying. To resolve this, we look at the return per episode, averaged over the first $X$ samples. More specifically, our metric consists of the average score per episode while learning for 20 hours (4,320,000 frames). In addition, we averaged the resulting number over 400 independent runs.

\begin{figure}[tb]
\begin{center}
\includegraphics[width=9cm]{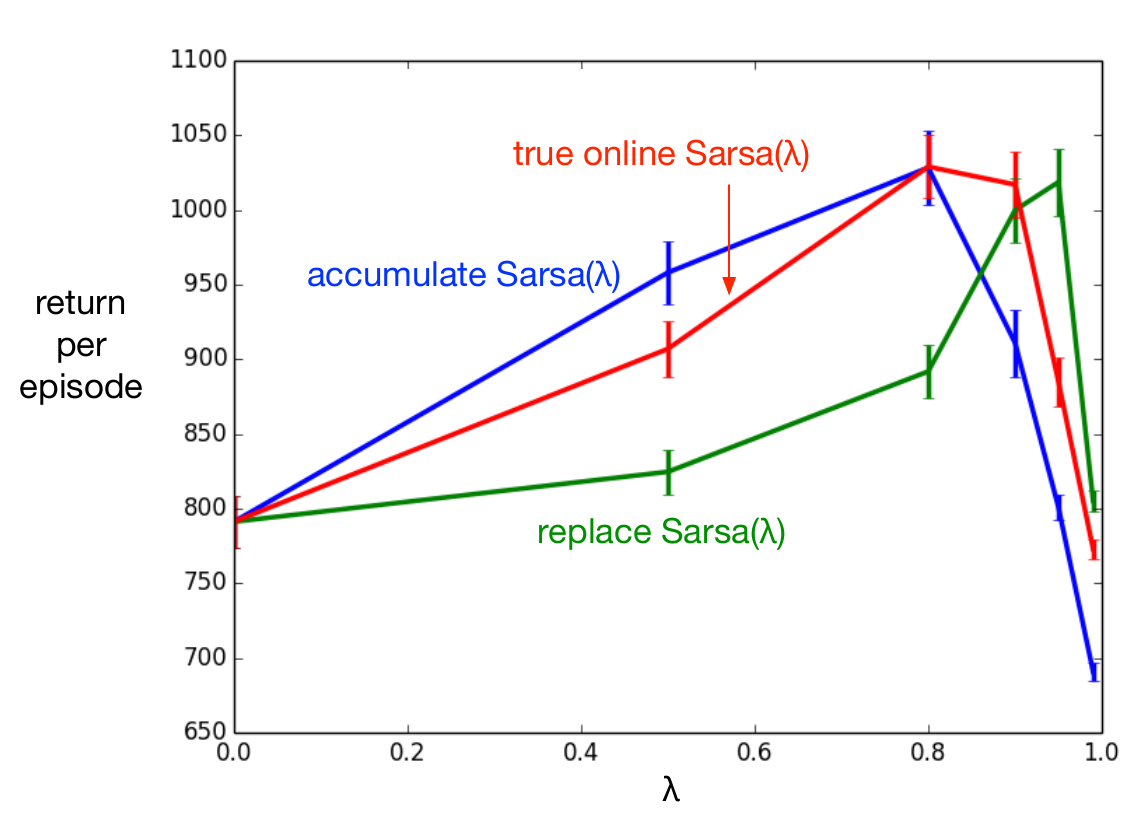}
\caption{Return per episode, averaged over the first 4,320,000 frames as well as 400 independent runs, as function of $\lambda$, at optimal $\alpha$, on the Asterix domain. }
\label{fig: results_asterix}
\end{center}
\end{figure}

As with the evaluation experiments, we performed a scan over the step-size $\alpha$ and the trace-decay parameter $\lambda$. Specifically, we looked at all combinations of $\alpha \in \{0.20, 0.50, 0.80$, $1.10, 1.40, 1.70, 2.00\}$ and $\lambda \in \{0.00, 0.50, 0.80, 0.90, 0.95, 0.99\}$ (these values were determined during a preliminary parameter sweep). We used a discount factor $\gamma = 0.999$ and $\epsilon$-greedy exploration with $\epsilon = 0.01$. The weight vector was initialized to the zero vector. Also, as \citet{bellemare:jair13}, we take an action at each 5 frames. This decreases the algorithms running time 
and avoids ``super-human'' reflexes. The results are shown in Figure \ref{fig: results_asterix}. On this domain, the optimal performance of all three versions of Sarsa($\lambda$) is similar.

Note that the way we evaluate a domain is computationally very expensive: we perform scans over $\lambda$ and $\alpha$, and use a large number of independent runs to get a low standard error. In the case of Asterix, this results in a total of $7 \cdot 6 \cdot 400 = 16, 800$ runs per method. This rigorous evaluation prohibits us unfortunately to run experiments on the full suite of ALE domains.

\subsection{Discussion}
\label{sec:discussion}

Figure \ref{fig:results summary} summarizes the performance of the different TD($\lambda$) versions on all evaluation domains. Specifically, it shows the error for each method at its best settings of $\alpha$ and $\lambda$. The error is normalized by dividing it by the error at $\lambda = 0$ (remember that all versions of TD($\lambda$) behave the same for $\lambda = 0$). 
Because  $\lambda = 0$ lies in the parameter range that is being optimized over, the normalized error can never be higher than 1. If for a method/domain the normalized error is equal to 1, this means that setting $\lambda$ higher than 0 either has no effect, or that the error gets worse. In either case, eligibility traces are not effective for that method/domain.

Overall, true online TD($\lambda$) is clearly better than accumulate TD($\lambda$) and replace TD($\lambda$) in terms of optimal performance. Specifically, for each considered domain/representation, the error for true online TD($\lambda$) is either smaller or equal to the error of accumulate/replace TD($\lambda$). This is especially impressive, given the wide variety of domains, and the fact that the computational overhead for true online TD($\lambda$) is small (see Section \ref{sec:algorithm} for details). 

The observed performance differences correspond well with the analysis from Section \ref{sec:relative performance}.
In particular, note that MRP (10, 3, 0.1) has less states than the other two MRPs, and hence the chance that the same feature has a non-zero value within a small time frame is larger. The analysis correctly predicts that this results in larger performance differences. Furthermore, MRP $(100,3,0)$ is less stochastic than MRP $(100,10, 0.1)$, and hence it has a smaller variance on the return. Also here, the experiments correspond with the analysis, which predicts that this results in a larger performance difference. 

On the Asterix domain, the performance of the three Sarsa($\lambda$) versions is similar. This is in accordance with the evaluation results, which showed that the size of the performance difference is domain dependent. In the worst case, the performance of the true online method is similar to that of the regular method.

\begin{figure}[tb]
\hspace{-2cm}
\includegraphics[width=17cm]{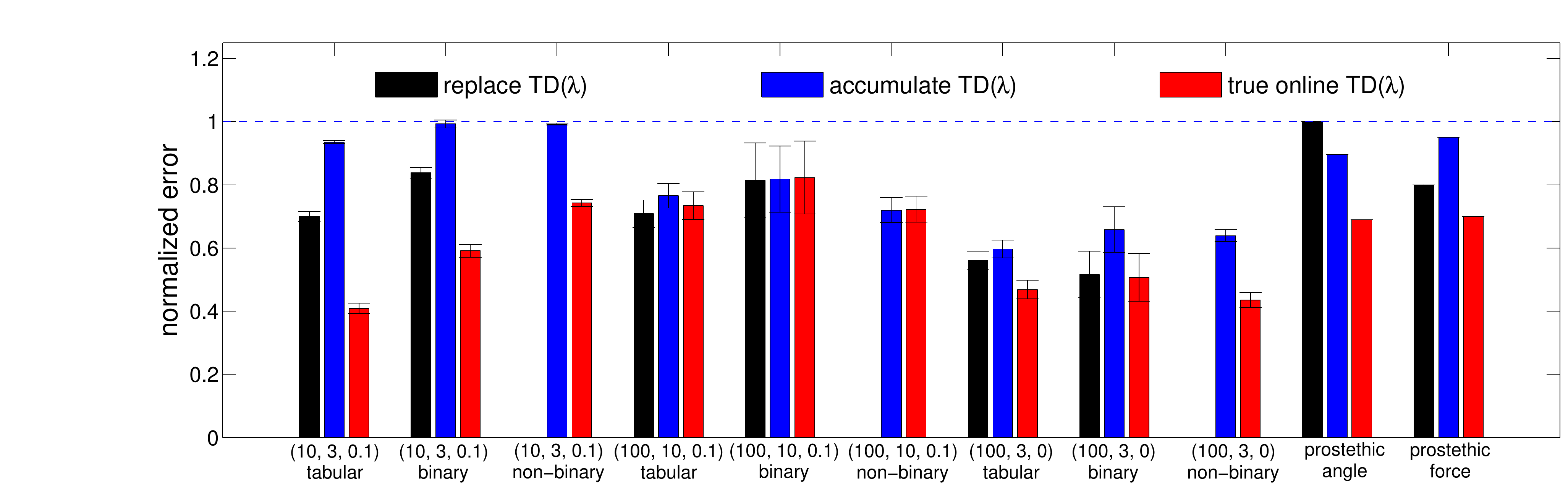}
\qquad
\caption{Summary of the evaluation results: error at optimal $(\alpha, \lambda)$-settings for all domains/representations, normalized with the TD(0) error.}
\label{fig:results summary}
\end{figure}

The optimal performance is not the only factor that determines how good a method is; what also matters is how easy it is to find this performance. The detailed plots in appendices  \ref{sec:detailed mrps} and \ref{sec:detailed myo} reveal that the parameter sensitivity of accumulate TD($\lambda$) is much higher than that of true online TD($\lambda$) or replace TD($\lambda$). This is clearly visible for MRP (10, 3, 0.1)  (Figure \ref{fig:randomMRP1}), as well as the experiments with the myoelectric prosthetic arm (Figure \ref{fig:myoelectricresults_all}).

There is one more thing to take away from the experiments. In MRP (10, 3, 0.1) with non-binary features, replace TD($\lambda$) is not applicable and accumulate TD($\lambda$) is ineffective. However, true online TD($\lambda$) was still able to obtain a considerable performance advantage with respect to TD(0). This demonstrates that true online TD($\lambda$) expands the set of domains/representations where eligibility traces are effective. This could potentially have far-reaching consequences. Specifically, using non-binary features becomes a lot more interesting. Replacing traces are not applicable to such representations, while accumulating traces can easily result in divergence of values. For true online TD($\lambda$), however, non-binary features are not necessarily more challenging than binary features. 
Exploring new, non-binary representations could potentially further improve the performance for true online TD($\lambda$) on domains such as the myoelectic prosthetic arm or the Asterix domain.

\section{Other True Online Methods}

In Appendix \ref{sec:derivation}, it is shown that the true online TD($\lambda$) equations can be derived directly from the online forward view equations. By using different online forward views, new true online methods can be derived. Sometimes, small changes in the forward view, like using a time-dependent step-size, can result in surprising changes in the true online equations.  In this section, we look at a number of such variations.

\subsection{True Online TD($\lambda$) with Time-Dependent Step-Size}
\label{sec:time-dependent step-size}

When using a time-dependent step-size in the base equation of the forward view (Equation \ref{eq:real-time update}) and deriving the update equations following the procedure from Appendix \ref{sec:derivation}, it turns out that a slightly different trace definition appears. We indicate this new trace using a `+' superscript: ${\bs e}^+$. For fixed step-size, this new trace definition is equal to:
$${\bs e}^+_t = \alpha {\bs e}_t\,, \qquad\mbox{ for all } t.$$
Of course, using ${\bs e}^+_t$ instead of ${\bs e}_t$ also changes the weight vector update slightly. Below, the full set of update equations is shown:
\begin{eqnarray*}
\delta_t &=& R_{t+1} + \gamma \w_t^\tr \x_{t+1}   - \w_{t}^\tr \x_{t} \,,\\
{\bs e}^+_t &=& \gamma\lambda {\bs e}^+_{t-1} +  \alpha_t \x_t -  \alpha_t \gamma \lambda \big(({\bs e}^+_{t-1})^\tr \,\x_t\big) \,\x_t \,,\\
\w_{t+1} &=&  \w_t + \delta_t\,{\bs e}^+_{t} + \big( \w_{t}^\tr \x_{t} - \w_{t - 1}^\tr \x_{t} \big) \big( {\bs e}^+_t -  \alpha_t \x_t \big) \,.
\end{eqnarray*}
In addition, ${\bs e}^+_{-1}:= 0$. We can simplify the weight update equation slightly, by using
$$\delta_t' = \delta_t +  \w_{t}^\tr \x_{t} - \w_{t - 1}^\tr \x_{t}\,,$$
which changes the update equations to:\footnote{These equations are the same as in the original true online paper \citep{vanseijen:icml14}.}
\begin{eqnarray}
\delta_t' &=& R_{t+1} + \gamma \w_t^\tr \x_{t+1}   - \w_{t - 1}^\tr \x_{t} \,,\\
{\bs e}^+_t &=& \gamma\lambda {\bs e}^+_{t-1} +  \alpha_t \x_t -  \alpha_t \gamma \lambda \big(({\bs e}^+_{t-1})^\tr \,\x_t\big) \,\x_t \,,\\
\w_{t+1} &=&  \w_t + \delta'_t\,{\bs e}^+_{t} - \alpha_t \big(\w_{t}^\tr \x_{t} - \w_{t - 1}^\tr \x_{t} \big) \x_t \,.\quad
\end{eqnarray}
Algorithm \ref{al:true TD(lambda)} shows the corresponding pseudocode. Of course, this pseudocode can also be used for constant step-size.

\begin{algorithm}[tb]
\begin{algorithmic}[0]
\STATE {\bf INPUT:} $\lambda, \w_{init}$, $\alpha_t$ for $t \geq 0$
\STATE $\w \leftarrow \w_{init}\,;\,\, t \leftarrow 0$
\STATE Loop (over episodes):
\STATE \qquad obtain initial $\x$
\STATE \qquad${\bs e}^+ \leftarrow {\bs 0}\,; \,\, {V}_{old} \leftarrow 0$
\STATE \qquad While terminal state is not reached, do:
\STATE \qquad\qquad obtain next feature vector $\x'$ and reward $R$
\STATE \qquad\qquad $V \leftarrow \w^\tr\x$
\STATE \qquad\qquad $V' \leftarrow \w^\tr\x'$
\STATE \qquad\qquad $\delta' \leftarrow R + \gamma\, V' - V_{old}$
\STATE \qquad\qquad ${\bs e}^+ \leftarrow   \gamma\lambda  {\bs e}^+  + \alpha_t  \x -   \alpha_t \gamma\lambda (({\bs e}^+)^\tr \x )\,\x$
\STATE \qquad\qquad $\w \leftarrow  \w + \delta' {\bs e}^+ -  \alpha_t (V - V_{old}) \x$
\STATE  \qquad\qquad $V_{old} \leftarrow V'$
\STATE \qquad\qquad $\x \leftarrow \x' $
\STATE \qquad\qquad $t \leftarrow t+1 $
\caption{true online TD($\lambda$) with time-dependent step-size}
\label{al:time-dependent true online TD(lambda)}
\end{algorithmic}
\end{algorithm}

\subsection{True Online Version of Watkins's Q($\lambda$)}

So far, we just considered \emph{on-policy} methods, that is, methods that evaluate a policy that is the same as the policy that generates the samples. However, the true online principle can also be applied to \emph{off-policy} methods, for which the evaluation policy is different from the behaviour policy. As a simple example, consider Watkins's Q($\lambda$) \citep{watkins:thesis89}. This is an off-policy method that evaluates the greedy policy given an arbitrary behaviour policy. It does this by combining accumulating traces with a TD error that uses the maximum state-action value of the successor state:
$$\delta_t = R_{t+1} + \max_a Q(S_t, a) -  Q(S_t, A_t)\,.$$
In addition, traces are reset to 0 whenever a non-greedy action is taken.

\begin{algorithm}[tb]
\begin{algorithmic}[0]
\STATE {\bf INPUT: $\alpha, \lambda, \gamma, \w_{init}, \Psi $}
\STATE $\w \leftarrow \w_{init}$
\STATE Loop (over episodes):
\STATE \qquad obtain initial state $S$
\STATE \qquad select action $A$ based on state $S$ \quad (for example $\epsilon$-greedy)
\STATE \qquad ${\bs \psi} \leftarrow $ features corresponding to $S, A$
\STATE \qquad ${\bs e} \leftarrow {\bs 0}\,;\,\, {Q}_{old} \leftarrow 0$
\STATE \qquad While terminal state has not been reached, do:
\STATE \qquad\qquad take action $A$, observe next state $S'$ and reward $R$
\STATE \qquad\qquad select action $A'$ based on state $S'$
\STATE \qquad\qquad $A^* \leftarrow \arg\max_a \w^\tr {\bs \psi}(S',a) $  (if $A'$ ties for the max, then $A^* \leftarrow A'$)
\STATE \qquad\qquad ${\bs \psi}' \leftarrow $ features corresponding to $S', A^*$   \quad(if $S'$ is terminal state, ${\bs \psi}'  \leftarrow {\bs 0}$)
\STATE \qquad\qquad $Q \leftarrow \w^\tr {\bs \psi} $
\STATE \qquad\qquad $Q' \leftarrow \w^\tr {\bs \psi'} $
\STATE \qquad\qquad $\delta \leftarrow R + \gamma\, Q' - Q$
\STATE \qquad\qquad $ {\bs e} \leftarrow  \gamma\lambda{\bs e} + {\bs \psi}  - \alpha  \gamma\lambda \big({\bs e}^\tr {\bs \psi}  \big)\,{\bs \psi} $
\STATE \qquad\qquad $\w \leftarrow  \w + \alpha (\delta + Q - Q_{old}) \,{\bs e} - \alpha (Q - Q_{old}) {\bs \psi} $
\STATE \qquad\qquad if  $A' \neq A^*$ :  ${\bs e} \leftarrow {\bs 0}$
\STATE  \qquad\qquad $Q_{old} \leftarrow Q'$
\STATE \qquad\qquad ${\bs \psi} \leftarrow {\bs \psi'}  \,;\,\,A \leftarrow A'$
\caption{true online version of Watkins's Q($\lambda$)}
\label{al:true online Watkins Q}
\end{algorithmic}
\end{algorithm}

From an online forward-view perspective, the strategy of Watkins's Q($\lambda$) method can be interpreted as a growing update target that stops growing once a non-greedy action is taken. Specifically, let $\tau$ be the first time step \emph{after} time step $t$ that a non-greedy action is taken, then the interim update target for time step $t$ can be defined as:
\begin{displaymath}
    U_t^h := (1 - \lambda) \sum_{n = 1}^{z-t-1} \lambda^{n-1} \tilde G_t^{(n)} + \lambda^{z-t-1} \tilde G_t^{(z-t)} \thinspace, \qquad z = min \{h, \tau \}\,,
\end{displaymath}
with
\begin{displaymath}
\tilde G_t^{\,(n)} = \sum_{k=1}^{n} \gamma^{k-1} R_{t+k} + \gamma^n  \max_a \w_{t+n-1}^\tr {\bs \psi}(S_{t+n},a)\,.
\end{displaymath}

Algorithm \ref{al:true online Watkins Q} shows the pseudocode for the true online method that corresponds with this update target definition. A problem with Watkins's Q($\lambda$) is that if the behaviour policy is very different from the greedy policy traces are reset very often, reducing the overall effect of the traces. In Section \ref{sec:related work}, we discuss more advanced off-policy methods. 

\subsection{Tabular True Online TD($\lambda$)}
\label{sec:tabular true online TD(lambda)}

Tabular features are a special case of linear function approximation. Hence, the update equations for true online TD($\lambda$) that  are presented so far also apply to the tabular case. However, we discuss it here separately, because the simplicity of this special case can provide extra insight. 

Rewriting the true online update equations (equations \ref{eq:to_delta_update} -- \ref{eq:to_weight_update}) for the special case of tabular features results in:
\begin{eqnarray*}
\delta_t &=& R_{t+1} + \gamma V_t(S_{t+1})   -  V_t(S_{t}) \,,\\
e_t(s) &=& \begin{cases} \gamma\lambda e_{t-1}(s)\,, & \mbox{ if }  s \neq S_t \,;\\ (1 - \alpha) \gamma\lambda  e_{t-1}(s) + 1\,, & \mbox{ if } s = S_t\,,  \end{cases}\\
V_{t+1}(s) &=& \begin{cases} V_{t}(s) + \alpha \big(\delta_t + V_t(S_t)- V_{t-1}(S_{t})\big)\,e_{t}(s)\,, & \mbox{ if }  s \neq S_t\,; \\
V_{t}(s) + \alpha \big(\delta_t + V_t(S_t)- V_{t-1}(S_{t})\big)\,e_{t}(s) - \alpha\big(V_t(S_t)- V_{t-1}(S_{t})\big) \,,& \mbox{ if }  s = S_t \,.\end{cases}
\end{eqnarray*}
What is interesting about the tabular case is that the dutch-trace update reduces to a particularly simple form. In fact, for the tabular case, a dutch-trace update is equal to the weighted average between an accumulating-trace update and a replacing-trace update, with the weight of the former $(1-\alpha)$ and the latter $\alpha$.
Algorithm \ref{al:tabular true online TD(lambda)} shows the corresponding pseudocode.

\begin{algorithm}[tb]
\begin{algorithmic}[0]
\STATE initialize $v(s)$  for all $s$
\STATE Loop (over episodes):
\STATE \qquad initialize $S$
\STATE \qquad $e(s) \leftarrow 0$ for all $s$
\STATE \qquad ${V}_{old} \leftarrow 0$
\STATE \qquad While $S$ is not terminal, do:
\STATE \qquad\qquad obtain next state $S'$ and reward $R$
\STATE \qquad\qquad $\Delta V \leftarrow V(S) - V_{old}$
\STATE \qquad\qquad $V_{old} \leftarrow V(S')$
\STATE \qquad\qquad $\delta \leftarrow R + \gamma\, V(S') - V(S)$
\STATE \qquad\qquad $e(S) \leftarrow (1-\alpha) e(S) + 1$
\STATE \qquad\qquad For all $s$:
\STATE \qquad\qquad\qquad $V(s) \leftarrow  V(s) + \alpha (\delta + \Delta V) e(s)$
\STATE \qquad\qquad\qquad $e(s) \leftarrow \gamma\lambda e(s)$
\STATE \qquad\qquad $V(S) \leftarrow V(S) - \alpha\Delta V$
\STATE \qquad\qquad $S \leftarrow S' $
\caption{tabular true online TD($\lambda$)}
\label{al:tabular true online TD(lambda)}
\end{algorithmic}
\end{algorithm}

\section{Related Work}
\label{sec:related work}

Since the first publication on true online TD($\lambda$) \citep{vanseijen:icml14}, several related papers have been published, extending the underlying concepts and improving the presentation. In sections \ref{sec:dutch traces}, \ref{sec:bv derivation} and \ref{sec:non-linear}, we review those papers. In Section \ref{sec: other variations}, we discuss other variations of TD($\lambda$).

\subsection{True Online Learning and Dutch Traces}
\label{sec:dutch traces}

As mentioned before, the traditional forward view, which is based on the $\lambda$-return, is inherently an offline forward view, because the $\lambda$-return is constructed from data up to the end of an episode. As a consequence, the work regarding equivalence between a forward view and a backward view traditionally focused on the final weight vector $\w_T$. This changed in 2014, when two papers introduced an \emph{online} forward view with a corresponding backward view that has an exact equivalence at each moment in time \citep{vanseijen:icml14, sutton:icml14}.
%One of the papers is `true online TD($\lambda$)' by Van Seijen \& Sutton; the other paper is `A new Q(lambda) with interim forward view and Monte Carlo equivalence' by Sutton et. al. 
While both papers introduced an online forward view, the two forward views presented are very different from each other. One difference is that the forward view introduced by van Seijen \& Sutton is an on-policy forward view, whereas the forward view by Sutton et al. is an off-policy forward view. However, there is an even more fundamental difference related to how the forward views are constructed. In particular, the forward view by van Seijen \& Sutton is constructed in such a way that at each moment in time the weight vector can be interpreted as the result of a sequence of updates of the form:
\begin{equation}
\w_{k+1}  =\w_k + \alpha \big( U_k -  \w_k^\tr\x_k \big)\x_k\,,  \qquad \mbox{for }0 \leq k < t\,.
\label{eq:on update}
\end{equation}
By contrast, the forward view by Sutton et al. gives the following interpretation:
\begin{equation}
\w_{t}  = \w_0 + \alpha\sum_{k=0}^{t-1} \delta_k \x_k,
\label{eq:off update}
\end{equation}
with $\delta_k$ some multi-step TD error. %\footnote{\citet{sutton:book98} refer to (\ref{eq:on update}) as an online update and to (\ref{eq:off update}) as an offline update.}
 Of course, the different forward views also result in different backward views. Whereas the backward view of Sutton et al. uses a generalized version of an accumulating trace, the backward view of van Seijen \& Sutton introduced a completely new type of trace. 
 
The advantage of a forward view based on (\ref{eq:on update}) instead of (\ref{eq:off update}) is that it results in much more stable updates. In particular, the sensitivity to divergence of accumulate TD($\lambda$) is a general side-effect of (\ref{eq:off update}), whereas (\ref{eq:on update}) is much more robust with respect to divergence. As a result, true online TD($\lambda$) not only has the property that it has an exact equivalence with an online forward view at all times, it consistently dominates TD($\lambda$) empirically.

The strong performance of true online TD($\lambda$) motivated \citet{vanhasselt:uai14} to construct an off-policy version of the forward view of true online TD($\lambda$). The corresponding backward view resulted in the algorithm true online GTD($\lambda$), which empirically outperforms GTD($\lambda$). They also introduced the term `dutch traces' for the new eligiblity trace.

Van Hasselt \& Sutton (\citeyear{vanhasselt:arxiv14}) showed that dutch traces are not only useful for TD learning. In an offline setting with no bootstrapping using dutch traces can result in certain computational advantages. To understand why, consider the Monte Carlo algorithm (MC),  which updates state values at the end of an episode using (\ref{eq:on update}), with the full return as update target. MC requires storing all the feature vectors and rewards during an episode, so the memory complexity is linear in the length of the episode. Moreover, the required computation time is distributed very unevenly: during an episode almost no computation is required, but at the end of an episode there is a huge peak in computation time due to all the updates that need to be performed. With dutch traces an alternative implementation can be made that results in the same final weight vector but that does not require storing all the feature vectors and where the required computation time is spread out evenly over all the time steps. Van Hasselt \& Sutton refer to this appealing property as span-independence: the memory and computation time required per time step is constant and independent of the span of the prediction.\footnote{The span of the prediction refers to the time difference between the first prediction and the moment its target is known (e.g., for episodic tasks it corresponds to the length of an episode).}

\subsection{Backward View Derivation}
\label{sec:bv derivation}

The task of finding an efficient backward view that corresponds exactly with a particular online forward view is not easy. Moreover, there is no guarantee that there exists an efficient implementation of a particular online forward view. Often, minor changes in the forward view determine whether or not an efficient backward view can be constructed. This created the desire to somehow automate the process of constructing an efficient backward view.

Van Seijen \& Sutton (\citeyear{vanseijen:icml14}) did not provide a direct derivation of the backward view update equations; they simply proved that the forward view and the backward view equations result in the same weight vectors. \citet{sutton:icml14} were the first to attempt to come up with a general strategy for deriving a backward view (although for forward views based on Equation \ref{eq:off update}). Van Hasselt et al. (\citeyear{vanhasselt:uai14}) took the approach of providing a theorem that proves equivalence between a general forward view and a corresponding general backward view. They showed that the forward/backward view of true online TD($\lambda$) is a special case of this general forward/backward view. They showed the same for the off-policy method that they introduced---true online GTD($\lambda$). Recently, \citet{mahmood:uai15} extended this theorem further by proving equivalence between an even more general forward view and backward view.

Furthermore, \cite{vanhasselt:arxiv14} derived backward views for a series of increasingly complex forward views. The derivation of the true online TD($\lambda$) equations in Appendix \ref{sec:derivation} is similar to those derivations.

\subsection{Extension to Non-Linear Function Approximation}
\label{sec:non-linear}

The linear update equations of the online forward view presented in Section \ref{sec:online forward view equations} can be easily extended to the case of non-linear function approximation. Unfortunately, it appears to be impossible to construct an efficient backward view with exact equivalence in the case of non-linear function approximation. The reason is that the derivation in Appendix \ref{sec:derivation} makes use of the fact that the gradient with respect to the value function is independent of the weight vector; this does not hold for non-linear function approximation.

Fortunately,  \citet{vanseijen:16} shows that many of the benefits of true online learning can also be achieved in the case of non-linear function approximation by using an alternative forward view (but still based on Equation \ref{eq:on update}). While this alternative forward view is not fully online (there is a delay in the updates), it can be implemented efficiently.

\subsection{Other Variations on TD($\lambda$)}
\label{sec: other variations}

Several variations on TD($\lambda$) other than those treated in this article have been suggested in the literature. \citet{schapire:ml96} introduced a variation of TD($\lambda$) for which upper and lower bounds on performance can be derived and proven. \citet{konidaris:nips11} introduced TD$_\gamma$, a parameter-free alternative to TD($\lambda$) based on a multi-step update target called the $\gamma$-return. TD$_\gamma$ is an offline algorithm with a computational cost proportional to the episode-length. Furthermore, \citet{thomas:nips15} proposed a method based on a multi-step update target, which they call the $\Omega$-return. The $\Omega$-return can account for the correlation of different length returns, something that both the $\lambda$-return and the $\gamma$-return cannot. However, it is expensive to compute and it is open question whether efficient approximations exist.

\section{Conclusions}

We tested the hypothesis that true online TD($\lambda$) (and true online Sarsa($\lambda$)) dominates TD($\lambda$) (and Sarsa($\lambda$)) with accumulating as well as with replacing traces by performing experiments over a wide range of domains. Our extensive results support this hypothesis.  
In terms of learning speed, true online TD($\lambda$) was often better, but never worse than TD($\lambda$) with either accumulating or replacing traces, across all domains/representations that we tried. Our analysis showed that especially on domains with non-sparse 
features and a relatively low variance on the return a large difference in learning speed can be expected. More generally, true online TD($\lambda$) has the advantage over TD($\lambda$) with replacing traces that it can be used with non-binary features, and it has the advantage over TD($\lambda$) with accumulating traces that it is less sensitive with respect to its parameters.
In terms of computation time, TD($\lambda$) has a slight advantage. In the worst case, true online TD($\lambda$) is twice as expensive. In the typical case of sparse features, it is only fractionally more expensive than TD($\lambda$). Memory requirements are the same. 
 Finally, we outlined an approach for deriving new true online methods, based on rewriting the equations of an online forward view. This may lead to new, interesting methods in the future.

\acks{The authors thank Hado van Hasselt for extensive discussions leading to the refinement of these ideas. Furthermore, the authors thank the anonymous reviewers for their valuable suggestions, resulting in a substantially improved presentation. This work was supported by grants from Alberta Innovates -- Technology Futures and the National Science and Engineering Research Council of Canada.
Computing resources were provided by Compute Canada through WestGrid.}

\appendix

\newpage

\section{Proof of Theorem 1}
\label{sec:proof}

\medskip
{\bf \noindent Theorem 1}
{\it \, Let $\w_0$ be the initial weight vector, $\w_{t}^{td}$ be the weight vector at time $t$ computed by accumulate TD($\lambda$), and $\w_{t}^{\lambda}$ be the weight vector at time $t$ computed by the online $\lambda$-return algorithm.
Furthermore, assume that $\sum_{i=0}^{t-1} \Delta_i^t \neq {\boldsymbol 0}$. Then, for all time steps $t$:
$$\frac{|| \w_{t}^{td}  - \w_{t}^{\lambda}  ||}{|| \w_{t}^{td}  - \w_0  ||} \rightarrow 0\,, \qquad\mbox{as $\,\,\,\alpha \rightarrow 0$}.$$
}

\begin{proof}
We prove the theorem by showing that $|| \w_{t}^{td}  - \w_{t}^{\lambda}  || / || \w_{t}^{td}  - \w_0  ||$ can be approximated by $\mathcal{O}(\alpha) / \big(C +  \mathcal{O}(\alpha)\big)$ as $\alpha \rightarrow 0$,  with $C > 0$. For readability, we will not use the `td' and `$\lambda$' superscripts; instead, we always use weights with double indices for the online $\lambda$-return algorithm and weights with single indices for accumulate TD($\lambda$). 

The update equations for accumulate TD($\lambda$) are:
\begin{eqnarray*}
\delta_t &=& R_{t+1} + \gamma \,\w_t^\tr \x_{t+1}   - \w_t^\tr \x_t\,, \\
{\bs e}_t &=& \gamma\lambda {\bs e}_{t-1} +  \x_t\,, \\
\w_{t+1} &=&  \w_t + \alpha \delta_t\,{\bs e}_{t}\,.
\end{eqnarray*}
By incremental substitution, we can write $\w_t$ directly in terms of $\w_0$:
\begin{eqnarray*}
{\w}_t &=& \w_0 + \alpha \sum_{j=0}^{t-1} \delta_j {\bs e}_j \,,\\
&=&  \w_0 + \alpha \sum_{j=0}^{t-1} \delta_j \sum_{i=0}^{j} (\gamma\lambda)^{j-i} \,  \x_i \,,\\
&=&  \w_0 + \alpha \sum_{j=0}^{t-1} \sum_{i=0}^{j} (\gamma\lambda)^{j-i} \delta_{j}\,  \x_i\,.
\end{eqnarray*}
Using the summation rule  $\sum_{j=k}^n \sum_{i=k}^j a_{i,j} = \sum_{i=k}^n \sum_{j=i}^n a_{i,j}$ we can rewrite this as:
\begin{equation}
 {\w}_t = \w_0 + \alpha \sum_{i=0}^{t-1} \sum_{j=i}^{t-1} (\gamma\lambda)^{j-i} \delta_{j}\,  \x_i \,.
 \label{eq:eq111}
 \end{equation}

As part of the derivation shown in Appendix \ref{sec:derivation}, we prove the following (see Equation \ref{eq:L_difference}):
$$G^{\lambda|h+1}_i = G^{\lambda|h}_i + (\lambda\gamma)^{h-i}  \delta_h' \,,$$
with
$$\delta_h' := R_{h+1} + \gamma \,\w_h^\tr \x_{h+1} - \w_{h-1}^\tr \x_{h}\,.$$
By applying this sequentially for $i+1 \leq h < t$, we can derive:
\begin{equation}
G^{\lambda | t}_i  =   G_i^{\lambda | i+1} + \sum_{j=i+1}^{t-1}   (\gamma\lambda)^{j-i} \delta'_{j}\,.
\label{eq:eq33}
\end{equation}
%Substituting $G_i^{\lambda | i+1} = \delta_i' + \x_{
Furthermore, $G_i^{\lambda | i+1}$ can be written as:
\begin{eqnarray*}
G_i^{\lambda | i+1} &=&  R_{i+1} + \gamma\w_i^\tr \x_{i+1} \,,\\
&=& R_{i+1} + \gamma\w_i^\tr \x_{i+1} - \w_{i-1}^\tr \x_{i} + \w_{i-1}^\tr \x_{i} \,,\\
&=& \delta'_i + \w_{i-1}^\tr \x_{i}\,.
\end{eqnarray*}
Substituting this in (\ref{eq:eq33}) yields:
$$G^{\lambda | t}_i  =   \w_{i-1}^\tr \x_{i}  + \sum_{j=i}^{t-1}   (\gamma\lambda)^{j-i} \delta'_{j} \,.$$
Using that $\delta_j' = \delta_j  + \w_{j}^\tr \x_{j} - \w_{j-1}^\tr \x_{j}$, it follows that
$$ \sum_{j=i}^{t-1}   (\gamma\lambda)^{j-i} \delta_{j} = G^{\lambda | t}_i  -  \w_{i-1}^\tr \x_{i} -  \sum_{j=i}^{t-1}   (\gamma\lambda)^{j-i} (\w_j - \w_{j-1})^\tr \x_j \,.$$
As $\alpha \rightarrow 0$, we can approximate this as:
\begin{eqnarray*}
 \sum_{j=i}^{t-1}   (\gamma\lambda)^{j-i} \delta_{j} &=& G^{\lambda | t}_i  -  \w_{i-1}^\tr \x_{i} + \mathcal{O}(\alpha)\,, \\
 &=& \bar G^{\lambda | t}_i  -  \w_{0}^\tr \x_{i} + \mathcal{O}(\alpha) \,,
 \end{eqnarray*}
with $\bar G^{\lambda | t}_i$ the interim $\lambda$-return that uses $\w_0$ for all value evaluations.
Substituting this in (\ref{eq:eq111}) yields:
\begin{equation}
 {\w}_t = \w_0 + \alpha \sum_{i=0}^{t-1} \big(\bar G^{\lambda | t}_i  -  \w_{0}^\tr \x_{i}  + \mathcal{O}(\alpha) \big)\,  \x_i\,.
 \label{eq:eq12}
\end{equation}

For the online $\lambda$-return algorithm, we can derive the following by sequential substitution of Equation (\ref{eq:real-time update}):
$$\w_t^t = \w_0 + \alpha \sum_{i=0}^{t-1} \Big( G_i^{\lambda|t} - (\w_i^t)^\tr \x_i\Big) \x_i\,.$$
As $\alpha \rightarrow 0$, we can approximate this as:
\begin{equation}
\w_t^t = \w_0 + \alpha \sum_{i=0}^{t-1} \Big(\bar G_i^{\lambda|t} - \w_0^\tr \x_i  +  \mathcal{O}(\alpha) \Big) \x_i\,.
\label{eq:1234}
\end{equation}

Combining (\ref{eq:eq12}) and (\ref{eq:1234}), it follows that as $\alpha \rightarrow 0$:
$$\frac{|| \w_{t}  - \w_{t}^{t}  ||}{|| \w_{t}  - \w_0  ||} = \frac{|| (\w_{t}  - \w_{t}^{t})/\alpha  ||}{|| (\w_{t}  - \w_0)/\alpha  || }  = \frac{\mathcal{O}(\alpha)}{C +  \mathcal{O}(\alpha)}\,,$$
with 
$$ C = \left|\left| \sum_{i=0}^{t-1} \Big(\bar G_i^{\lambda|t} - \w_0^\tr \x_i\Big) \x_i  \right|\right|\ = \left|\left| \sum_{i=0}^{t-1} \Delta_i^t \right|\right|\,.$$
From the condition $\sum_{i=0}^{t-1} \Delta_i^t \neq {\boldsymbol 0}$ it follows that $C > 0$.
\end{proof}

\section{Derivation Update Equations}
\label{sec:derivation}

In this subsection, we derive the update equations of true online TD($\lambda$) directly from the online forward view, defined by equations (\ref{eq:interim lambda return1}) and  (\ref{eq:real-time update})  (and $\w_0^t := \w_{init}$). The derivation is based on expressing $\w_{t+1}^{t+1}$ in terms of $\w_t^t$.

We start by writing $\w_t^t$ directly in terms of the initial weight vector and the interim $\lambda$-returns. First, we rewrite  (\ref{eq:real-time update}) as:
$$ \w_{k+1}^t = (\textbf{I} - \alpha\x_k \x_k^\tr) \,\w_k^t + \alpha\, \x_k G_k^{\lambda|t} \,,$$
with $\textbf{I}$ the identity matrix.
Now, consider $\w_k^t$ for $k=1$ and $k=2$:
\begin{eqnarray*}
\w_1^t &=& (\textbf{I} - \alpha \x_0\x_0^\tr) \w_{init} + \alpha \x_0 G_0^{\lambda|t}\,,\\
\w_2^t &=& (\textbf{I} - \alpha \x_1\x_1^\tr) \w_1^t + \alpha \x_1 G_1^{\lambda|t} \,,\\
&=& (\textbf{I} - \alpha \x_1\x_1^\tr) (\textbf{I} - \alpha \x_0\x_0^\tr) \w_{init} + \alpha (\textbf{I} - \alpha \x_1\x_1^\tr) \x_0 G_0^{\lambda|t} + \alpha \x_1 G_1^{\lambda|t}\,.
\end{eqnarray*}
For general $k \leq t$, we can write:
$$\w_k^t = \textbf{A}^{k-1}_0\, \w_{init} + \alpha \sum_{i=0}^{k-1}  \textbf{A}^{k-1}_{i+1}\, \x_{i} G_{i}^{\lambda|t}\thinspace,$$
where $\textbf{A}^j_i$ is defined as:
$$\textbf{A}^j_i := (\textbf{I} - \alpha \x_j \x_j^\tr )  (\textbf{I} - \alpha \x_{j-1} \x_{j-1}^\tr ) \dots  (\textbf{I} - \alpha \x_i \x_i^\tr ),  \quad \mbox{ for } j \geq i\thinspace,$$
and $\textbf{A}^j_{j+1} := \textbf{I}$.  We are now able to express $\w_t^t$ as:
\begin{equation}
\w_t^t = \textbf{A}^{t-1}_{0}\, \w_{init} + \alpha \sum_{i=0}^{t-1}  \textbf{A}^{t-1}_{i+1}\, \x_{i} G_{i}^{\lambda|t}\thinspace, \label{eq:w_kh}
\end{equation}
Because for the derivation of true online TD($\lambda$), we only need (\ref{eq:w_kh}) and the definition of $G_i^{\lambda|t}$, we can drop the double indices for the weight vectors and use $\w_t := \w_t^t$.

We now derive a compact expression for the difference $G_{i}^{\lambda|t+1}  - G_{i}^{\lambda|t}$: %For readability, we use $C_{k} := \w_{k-1}^\tr\x_{k}$:
\begin{eqnarray*}
G^{\lambda|t+1}_i - G^{\lambda|t}_i &=& (1 - \lambda) \sum_{n = 1}^{t-i} \lambda^{n-1} G_i^{(n)} + \lambda^{t-i} G_i^{(t+1-i)} \,,\\
&& -\,\, (1 - \lambda) \sum_{n = 1}^{t-i-1} \lambda^{n-1} G_i^{(n)} - \lambda^{t-i-1} G_i^{(t-i)} \,,\\
&=& (1 - \lambda) \lambda^{t-i-1} G_i^{(t-i)} + \lambda^{t-i} G_i^{(t+1-i)}  - \lambda^{t-i-1} G_i^{(t-i)}  \,,\\
&=&  \lambda^{t-i} G_i^{(t+1-i)}  -  \lambda^{t-i} G_i^{(t-i)} \,,\\
&=& \lambda^{t-i} \Big(G_i^{(t+1-i)} -  G_i^{(t-i)}\Big)\,,\\
&=&   \lambda^{t-i} \Big( \sum_{k=1}^{t+1-i} \gamma^{k-1} R_{i+k} + \gamma^{t+1-i} \w_t^\tr \x_{t+1}   -  \sum_{k=1}^{t-i} \gamma^{k-1} R_{i+k} - \gamma^{t-i} \w_{t-1}^\tr \x_{t}   \Big)\,,\\
&=&  \lambda^{t-i} \Big( \gamma^{t-i} R_{t+1} +  \gamma^{t+1-i} \w_t^\tr \x_{t+1} - \gamma^{t-i} \w_{t-1}^\tr \x_{t}  \Big)\,,\\
&=& (\lambda\gamma)^{t-i} \Big( R_{t+1} + \gamma \,\w_t^\tr \x_{t+1} - \w_{t-1}^\tr \x_{t} \Big)\,.\\
\end{eqnarray*}
Note that the difference $G^{\lambda|t+1}_i - G^{\lambda|t}_i$ is naturally expressed using a term that looks like a TD error but with a modified time step. We call this the modified TD error, $\delta'_t$:
$$\delta_t' := R_{t+1} + \gamma \,\w_t^\tr \x_{t+1} - \w_{t-1}^\tr \x_{t}.$$
The modified TD error relates to the regular TD error, $\delta_t$, as follows:
\begin{eqnarray}
\delta'_t &=& R_{t+1} + \gamma \,\w_t^\tr \x_{t+1} - \w_{t-1}^\tr \x_{t} \,,\nonumber\\
&=& R_{t+1} + \gamma \,\w_t^\tr \x_{t+1} - \w_{t}^\tr \x_{t} + \w_{t}^\tr \x_{t} - \w_{t-1}^\tr \x_{t} \,,\nonumber\\
&=& \delta_t + \w_{t}^\tr \x_{t} - \w_{t-1}^\tr \x_{t}\thinspace. \label{eq:modified_delta_relation}
\end{eqnarray}
Using $\delta'_t$, the difference $G^{\lambda|t+1}_i - G^{\lambda|t}_i$ can be compactly written as:
\begin{equation}
G^{\lambda|t+1}_i - G^{\lambda|t}_i  = (\lambda\gamma)^{t-i}  \delta_t' \,.\label{eq:L_difference}
\end{equation}

To get the update rule,  $\w_{t+1}$ has to be expressed in terms of $ \w_{t}$. This is done below, using (\ref{eq:w_kh}), (\ref{eq:modified_delta_relation}) and (\ref{eq:L_difference}):
\begin{eqnarray}
\w_{t+1} &=& \textbf{A}_{0}^t\, \w_0 + \alpha \sum_{i=0}^{t}  \textbf{A}^t_{i+1} \, \x_{i} G_{i}^{\lambda| t+1} \,,\nonumber \\
&=& \textbf{A}^t_{0}  \w_0 + \alpha \sum_{i=0}^{t-1}  \textbf{A}^t_{i+1}  \x_{i} G_{i}^{\lambda | t+1} + \alpha \x_{t} G_{t}^{\lambda|t+1} \,, \nonumber \\
&=& \textbf{A}^t_{0}  \w_0 + \alpha \sum_{i=0}^{t-1}  \textbf{A}^t_{i+1}  \x_{i} G_{i}^{\lambda | t}  + \alpha \sum_{i=0}^{t-1}  \textbf{A}^t_{i+1}  \x_{i} \big(G_{i}^{\lambda| t+1} - G_{i}^{\lambda| t}\big)
+ \alpha \x_{t} G_{t}^{\lambda | t+1} \,,\nonumber \\
&=& (\textbf{I} - \alpha \x_t \x_t^\tr)  \Big(\textbf{A}^{t-1}_{0} \, \w_0 + \alpha \sum_{i=0}^{t-1}  \textbf{A}^{t-1}_{i+1} \, \x_{i} G_{i}^{\lambda|t}\Big) \nonumber \\
&&  + \,\alpha \sum_{i=0}^{t-1}  \textbf{A}^t_{i+1} \x_{i} \,\Big(G_{i}^{\lambda | t+1} - G_{i}^{\lambda | t}\Big) + \alpha \x_{t} G_{t}^{\lambda | t+1} \,,\nonumber \\
&=& (\textbf{I} - \alpha \x_t \x_t^\tr)  \,\w_t  + \alpha \sum_{i=0}^{t-1}  \textbf{A}^t_{i+1}  \x_{i} \Big(G_{i}^{\lambda | t+1} - G_{i}^{\lambda | t}\Big) + \alpha \x_{t} G_{t}^{\lambda | t+1} \,, \nonumber \\
&=& (\textbf{I} - \alpha \x_t \x_t^\tr)  \w_t +  \alpha  \sum_{i=0}^{t-1}  \textbf{A}^t_{i+1}  \x_{i} (\gamma\lambda)^{t-i} \delta'_t  + \alpha \x_{t} \big( R_{t+1} + \gamma {\w_t}^\tr \x_{t+1}\big)\,,\nonumber \\
&=& \w_t +  \alpha  \sum_{i=0}^{t-1}  \textbf{A}^t_{i+1} \x_{i} (\gamma\lambda)^{t-i} \delta'_t + 
\alpha \x_{t} \big( R_{t+1} + \gamma {\w_t}^\tr \x_{t+1} - \w_t^\tr \x_t\big) \,,\nonumber \\
&=& \w_t +  \alpha  \sum_{i=0}^{t-1}  \textbf{A}^t_{i+1} \x_{i} (\gamma\lambda)^{t-i} \delta'_t \nonumber \\
&& +\, \alpha \x_{t} \big( R_{t+1} + \gamma {\w_t}^\tr \x_{t+1} - \w_{t-1}^\tr \x_t + \w_{t-1}^\tr \x_t - \w_t^\tr \x_t\big) \,,\nonumber \\
&=& \w_t +  \alpha  \sum_{i=0}^{t-1}  \textbf{A}^t_{i+1}  \x_{i} (\gamma\lambda)^{t-i} \delta'_t + \alpha \x_{t} \delta'_t - \alpha \big(\w_{t}^\tr \x_t - \w_{t-1}^\tr \x_t \big)\x_{t} \,, \nonumber \\
&=& \w_t +  \alpha  \sum_{i=0}^{t}  \textbf{A}^t_{i+1}  \x_{i} (\gamma\lambda)^{t-i} \delta'_t - \alpha  \big(\w_{t}^\tr \x_t - \w_{t-1}^\tr \x_t \big)\x_{t} \,,\nonumber \\
&=& \w_{t} + \alpha {\bs e}_t \delta'_t - \alpha  \big(\w_{t}^\tr \x_t - \w_{t-1}^\tr \x_t \big)\x_{t}\,, \qquad\qquad\mbox{ with }{\bs e}_t := \sum_{i=0}^{t}  
\textbf{A}^t_{i+1}  \x_{i}  (\gamma\lambda)^{t-i} \,,\nonumber \\
&=& \w_{t} + \alpha {\bs e}_t \big( \delta_t + \w_{t}^\tr \x_{t} - \w_{t-1}^\tr \x_{t} \big) - \alpha \big(\w_{t}^\tr \x_t - \w_{t-1}^\tr \x_t \big) \x_{t} \,,\nonumber \\
&=& \w_{t} + \alpha {\bs e}_t \delta_t + \alpha \big(\w_{t}^\tr \x_t - \w_{t-1}^\tr \x_t \big) \big({\bs e}_t - \x_{t} \big) \,.\label{eq:TO_theta_update}
\end{eqnarray}
\vspace{3cm}

We now have the update rule for $\w_t$, in addition to an explicit definition of ${\bs e}_t$. Next, using this explicit definition, we derive an update rule to compute ${\bs e}_t$ from  ${\bs e}_{t-1}$:
\begin{eqnarray}
{\bs e}_t &=& \sum_{i=0}^{t}  \textbf{A}^t_{i+1}  \x_{i} (\gamma\lambda)^{t-i} \,,\nonumber \\
&=&  \sum_{i=0}^{t-1}  \textbf{A}^t_{i+1}  \x_{i} (\gamma\lambda)^{t-i} +  \x_{t} \,,\nonumber \\
&=& (\textbf{I} - \alpha \x_t \x_t^\tr)  \gamma\lambda  \sum_{i=0}^{t-1} \textbf{A}^{t-1}_{i+1} \, \x_{i} (\gamma\lambda)^{t-i-1} +  \x_{t} \,,\nonumber \\
&=& (\textbf{I} - \alpha \x_t \x_t^\tr)  \gamma\lambda {\bs e}_{t-1} +  \x_{t} \,,\nonumber \\
&=& \gamma\lambda {\bs e}_{t-1} +  \x_{t} - \alpha \gamma\lambda ({\bs e}_{t-1}^\tr \x_t) \x_t  \,.\label{eq:TO_e_update}
\end{eqnarray}
Equations (\ref{eq:TO_theta_update}) and (\ref{eq:TO_e_update}), together with the definition of $\delta_t$, form the true online TD($\lambda$) update equations.

\newpage
\section{Detailed Results Random MRPs}
\label{sec:detailed mrps}

\begin{figure}[h]
\centering
\includegraphics[width=\columnwidth]{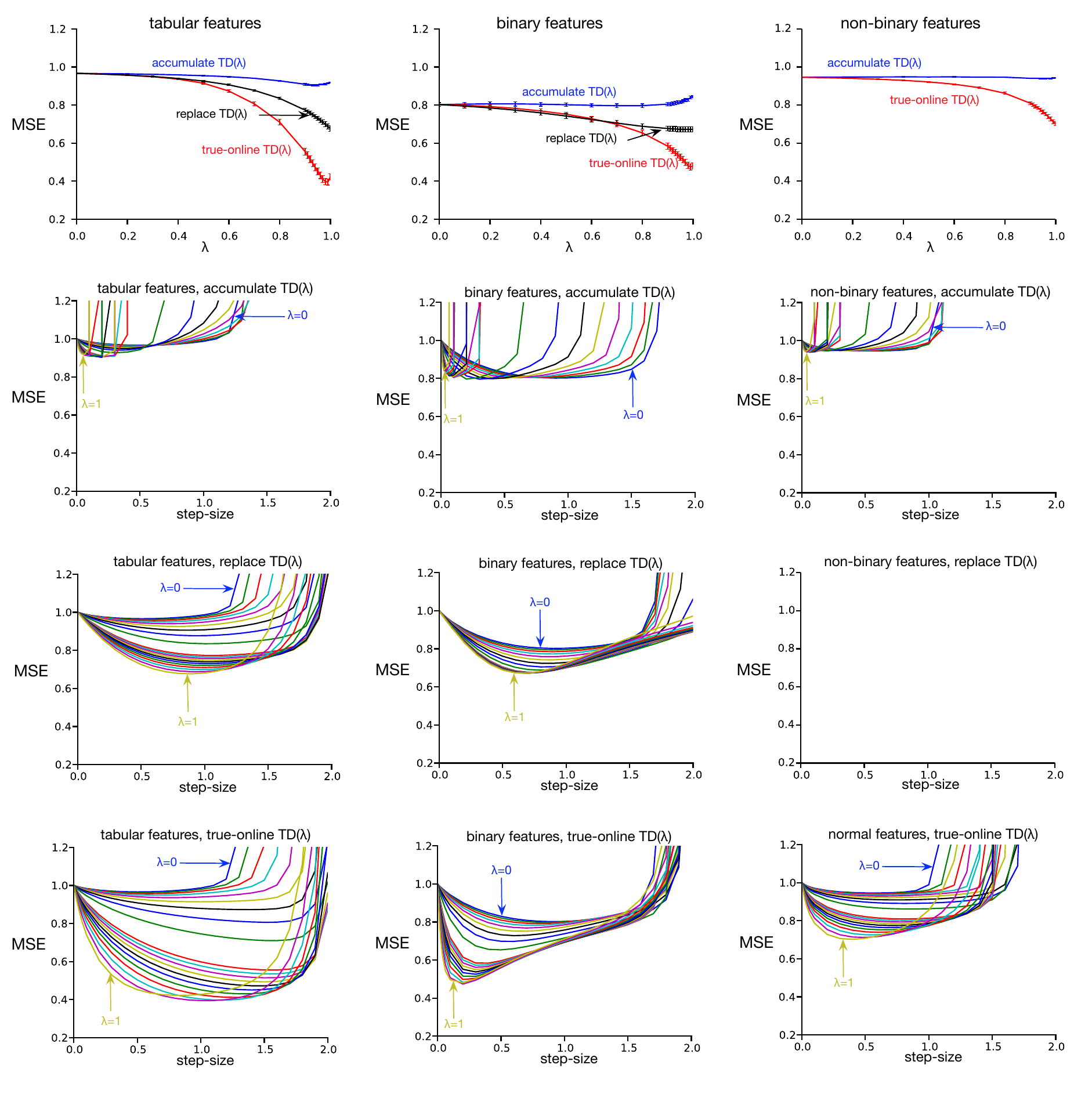}
\caption{Results on a random MRP with $k = 10$, $b = 3$ and $\sigma = 0.1$. MSE is the mean squared error averaged over the first 100 time steps, as well as 50 runs, and normalized using the initial error. The top graphs summarize the results from the graphs below it; they show the MSE error, for each $\lambda$, at the best step-size.}
\label{fig:randomMRP1}
\end{figure}

\newpage
\begin{figure}[h]
\centering
\includegraphics[width=\columnwidth]{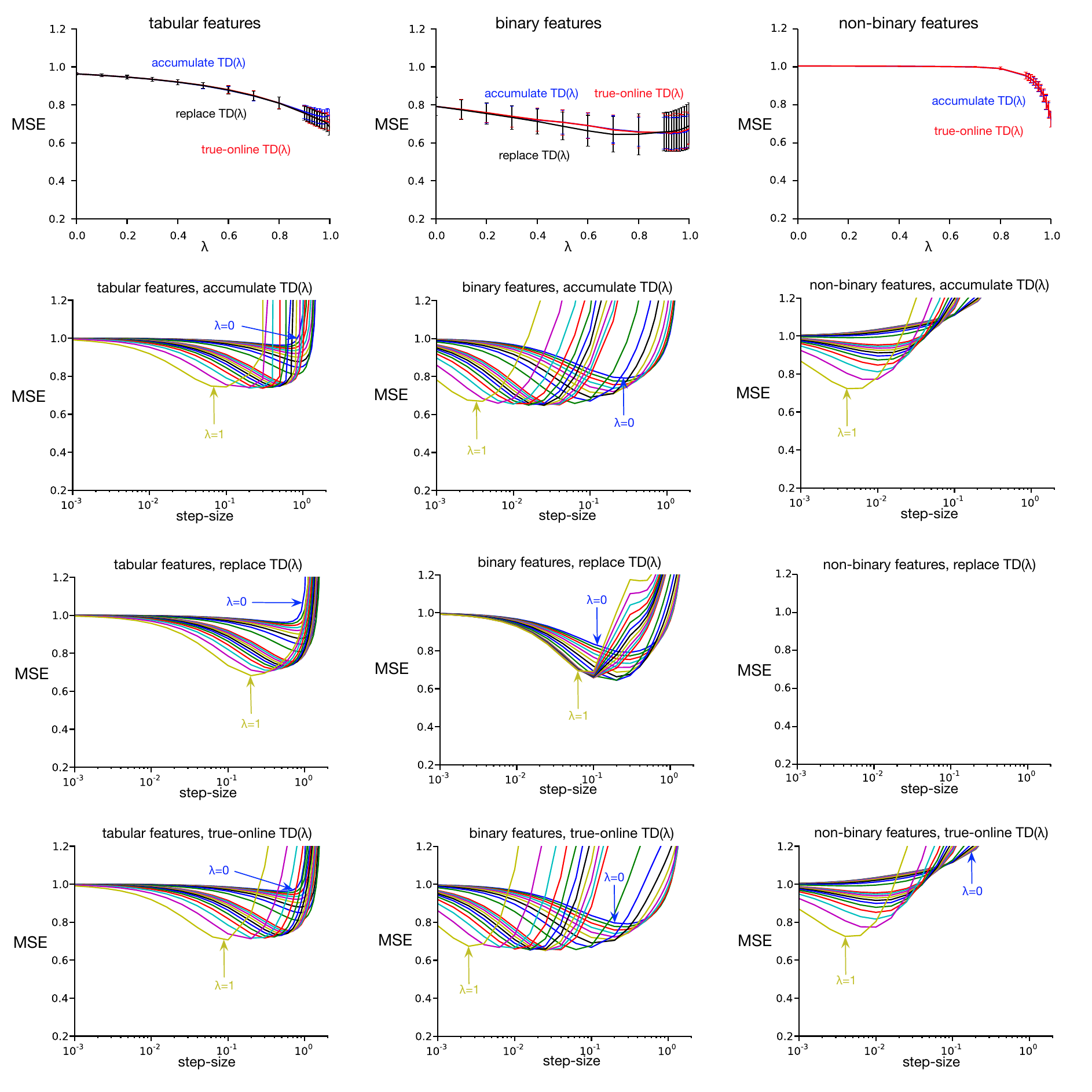}
\caption{Results on a random MRP with $k = 100$, $b = 10$ and $\sigma = 0.1$. MSE is the mean squared error averaged over the first 1000 time steps, as well as 50 runs, and normalized using the initial error.}
\label{fig:randomMRP2}
\end{figure}

\newpage
\begin{figure}[h]
\centering
\includegraphics[width=\columnwidth]{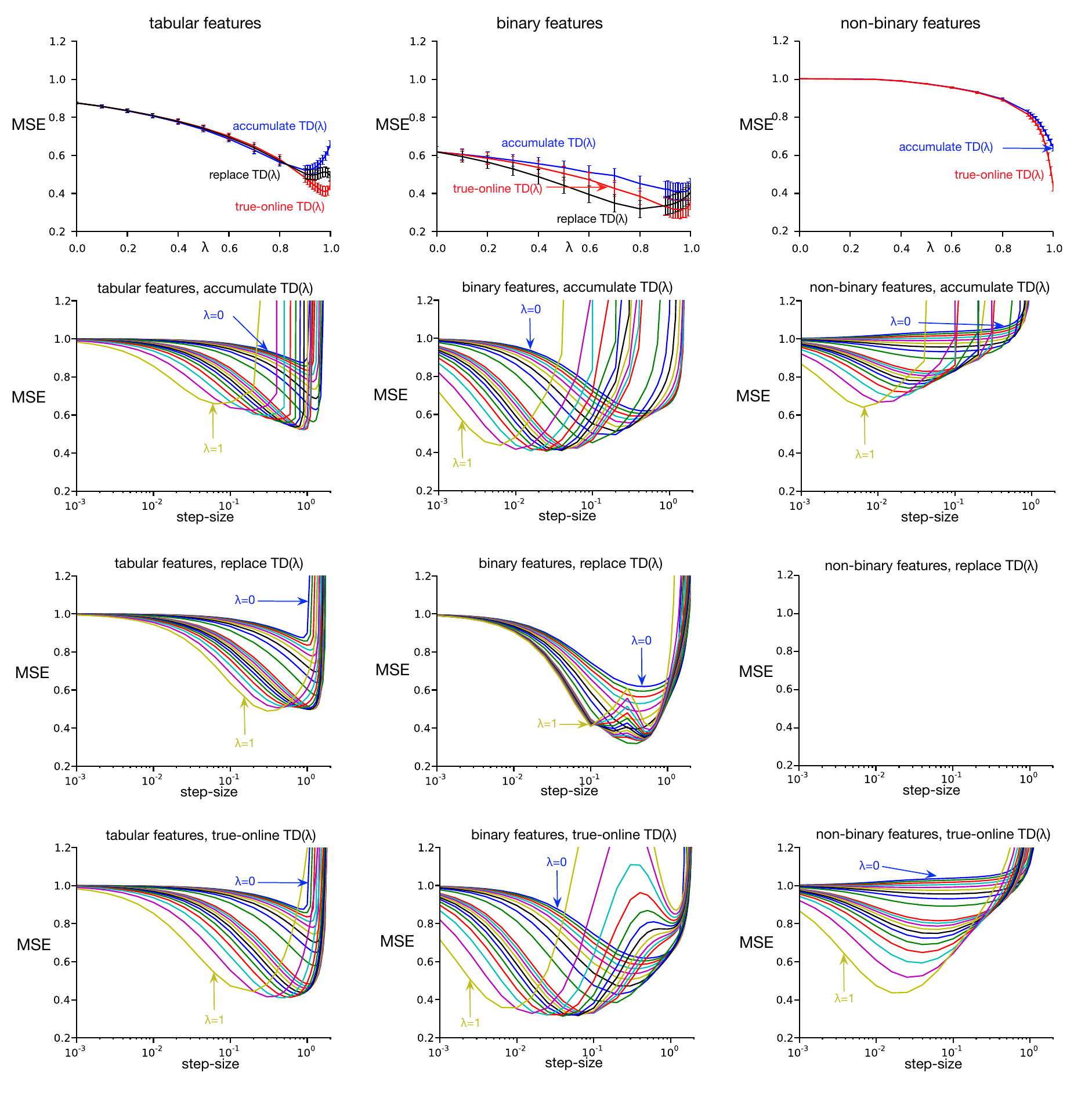}
\caption{Results on a random MRP with $k = 100$, $b = 3$ and $\sigma = 0$. MSE is the mean squared error averaged over the first 1000 time steps, as well as 50 runs, and normalized using the initial error.}
\label{fig:randomMRP3}
\end{figure}

\newpage
\section{Detailed Results for Myoelectric Prosthetic Arm}
\label{sec:detailed myo}

\begin{figure}[htbp]
\begin{center}
\includegraphics[width=4in]{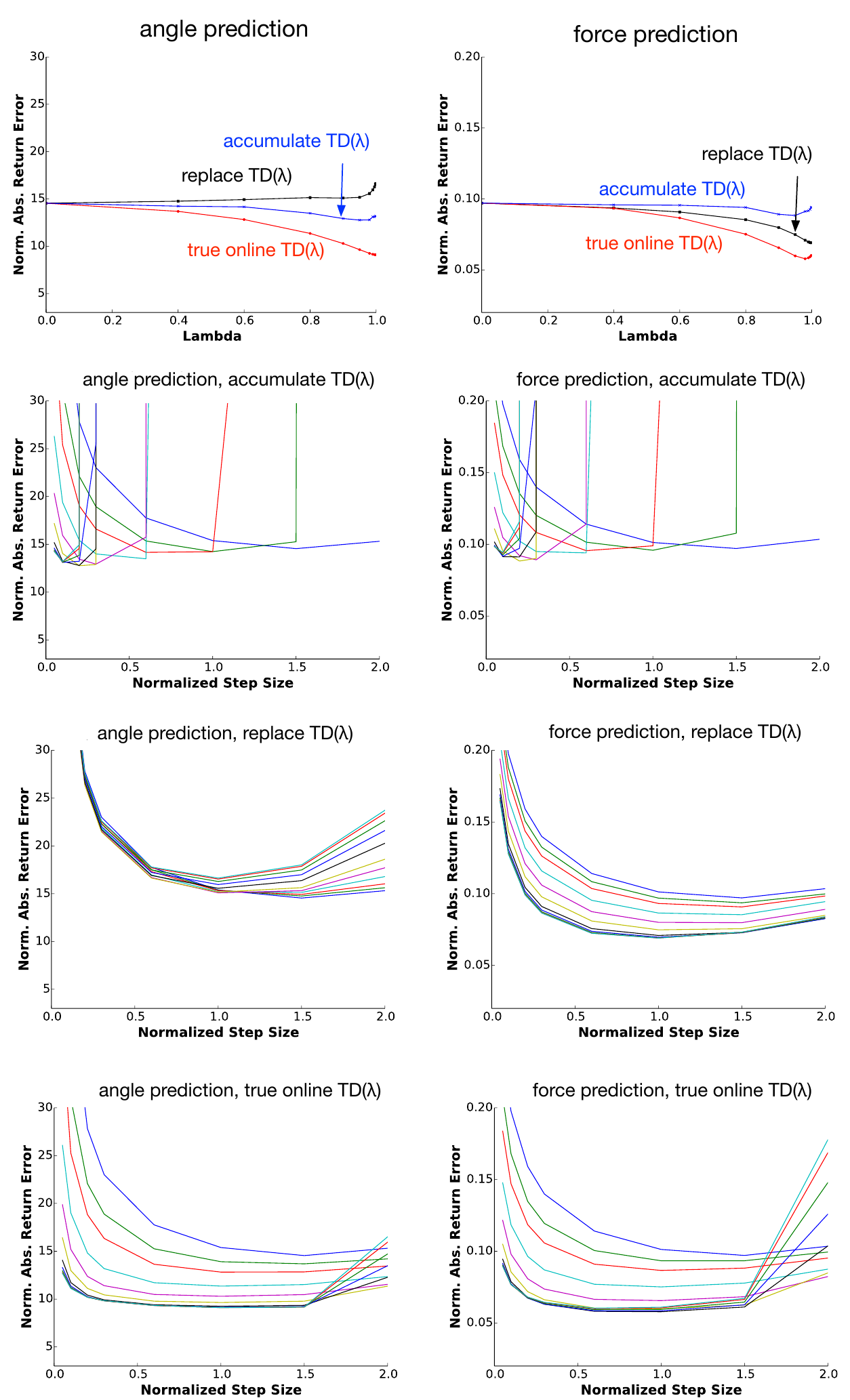}\\
\caption{Results on prosthetic data from the single amputee subject described in \cite{pilarski:ra13}, for the prediction of servo motor angle ({\em left column}) and grip force ({\em right column}) as recorded from the amputee's myoelectrically controlled robot arm during a grasping task.}
\label{fig:myoelectricresults_all}
\end{center}
\end{figure}

\end{document}